\documentclass{article} 
\usepackage{iclr2027_conference,times}

\newcommand{\bmu}{\boldsymbol{\mu}}
\usepackage{amsmath,amsfonts,bm}

\def\eqref#1{equation~\ref{#1}}

\def\1{\bm{1}}

\DeclareMathAlphabet{\mathsfit}{\encodingdefault}{\sfdefault}{m}{sl}
\SetMathAlphabet{\mathsfit}{bold}{\encodingdefault}{\sfdefault}{bx}{n}

\newcommand{\h}{\mathbf{h}}

\newcommand{\R}{\mathbb{R}}

\DeclareMathOperator*{\argmax}{arg\,max}
\DeclareMathOperator*{\argmin}{arg\,min}

\renewcommand{\R}{\mathbf{R}} 
\usepackage{url}

\usepackage{amsmath,amssymb,amsfonts}
\usepackage{graphicx}
\usepackage{booktabs}
\usepackage{colortbl} 
\usepackage{array}
\usepackage{multirow}
\usepackage{bm}
\usepackage{enumitem}
\usepackage{algorithm}
\usepackage{algpseudocode}
\usepackage{placeins}
\usepackage{amsthm}
\newtheorem{proposition}{Proposition}[section]
\newtheorem{remark}[proposition]{Remark}
\usepackage{float}
\usepackage{hyperref}
\graphicspath{{figures/}}


\iclrfinalcopy
\AtBeginDocument{\lhead{}\renewcommand{\headrulewidth}{0pt}}

\title{\Large Fair Fact-Checking: Closing the Cross-Lingual Gap\\[0.15em]
in LLM Factual Judgement with \textsc{RoSh}}

\author{
  \textbf{Muhammad Ahmad}$^{1,2}$ \quad
  \textbf{Fatemeh Seyedin}$^{3}$ \quad
  \textbf{Adrian Weller}$^{4}$ \\[0.25em]
  \textbf{Dongwon Lee}$^{5}$ \quad
  \textbf{Mahmoudreza Babaei}$^{2,3,6}$ \\[0.8em]
  \normalsize
  $^{1}$Department of Computer Science, Shifa Tameer-e-Millat University, Islamabad, Pakistan \\
  $^{2}$BRAINS, Brandenburg Research Center for Applied Intelligent Systems, Potsdam, Germany \\
  $^{3}$Max Planck Institute for Security and Privacy, Bochum, Germany \\
  $^{4}$University of Cambridge, United Kingdom \\
  $^{5}$The Pennsylvania State University, University Park, PA, USA \\
  $^{6}$GISMA University of Applied Sciences, Potsdam, Germany \\[0.5em]
}

\begin{document}
\lhead{\small Preprint. Under review.}
\maketitle

\begin{abstract}

Misinformation on social media remains a critical problem, and more and more people settle it by asking a language model instead of a fact checker, about anything from a claim circulating on a platform to a plain question of fact. Whether models judge such claims reliably
is still debated; whether they judge them equally well in every language people ask in has gone almost unasked. We test eight models from five families, 3B to 70B, on 1,500 claims that mix fact-checked news with encyclopedic facts and exist in identical form in eight languages. On every model English is judged better than the other languages, and the gap is widest on the smallest models, where Llama-3B on Arabic is no better than guessing. Existing remedies either retrain the model on more multilingual data or fit an unconstrained map between language representations, and neither asks whether the model already holds the answer and simply fails to say it. It largely does: a linear probe recovers the truth from the very activations the model fails to express. We propose RoSh, a per-language shift and rotation of the residual stream, computed in closed form at three layers, with no training and no weight modified. It improves every model and closes 75\% of the gap on average, and it helps most where the model was worst: the two smallest models end close to English, Arabic on Llama-3B goes from chance to nearly the English level, and a fifth fewer of the claims answered correctly in English are lost in translation. What remains is no longer a read-out failure: afterwards the head recovers at least as much of what is encoded outside English as it does in English.
An unconstrained map fitted on the same pairs falls below the untouched baseline, so the orthogonality constraint is doing the work, and every model clears a scrambled-correspondence control and ten further controls. On the two benchmarks of the closest inference-time method, latent-space intervention, run with its own data and metric code, RoSh's gains are five to thirteen times larger.
\end{abstract}

\section{Introduction}

Social media platforms have been criticised for years for letting false stories spread unchecked, and the damage does not stay online. It has been traced into elections \citep{allcott2017social} and, during the pandemic, into public health. Because false stories also diffuse faster and reach wider audiences than true ones, fact-checks typically arrive long after public attention has shifted \citep{vosoughi2018spread}. In response, platforms partnered with independent fact-checking organisations such as Snopes, PolitiFact, and Full Fact to suppress flagged content in user feeds. However, this approach assumes that users actively seek out verification, and increasingly they do not. Across 48 markets, 10\% of adults now use an AI chatbot for news every week, up from 7\% a year earlier, and 16\% among under-35s \citep{reuters2026}.

A chatbot asked about a claim becomes the fact checker, and unlike a fact-checking desk it answers in whatever language it is asked: one model serves all 48 of those markets, in many languages. That makes a new question pressing: is the model's verdict on a claim the same whichever language the claim arrives in?
To find out, we evaluated eight language models on a straightforward task: given a factual claim with a known truth value, we queried the same model in English and in another target language to test whether its verdict remained consistent. For Llama-3B, consistency broke down sharply. The model achieved an AUROC of $0.857$ in English but dropped to $0.523$ in Arabic, marginally better than random guessing. Similarly, Mistral-7B scored $0.834$ in English compared to $0.646$ in Portuguese.
Across an evaluation of eight models spanning five architecture families ($3\text{B}$ to $70\text{B}$ parameters) and $1{,}500$ claims held identical across eight languages, English consistently outperforms all other languages by an average of $+0.046$ AUROC, a gap that widens to $+0.116$ and $+0.130$ on the two smallest models. Because every claim is evaluated symmetrically across all languages, this disparity cannot be attributed to differences in prompt difficulty. The quality of the verdict depends on the language you happen to speak.

Two things make that worse than a benchmark gap. The users it fails have the fewest alternatives, since where the major services are restricted or unaffordable people fall back on regional systems trained on far less data, and nothing in the answer reveals the problem.
The judge being replaced was no more reliable: given stories fact-checkers had already settled, ordinary readers produced verdicts tracking their politics rather than the evidence \citep{babaei2021analyzing}. A model may have a perception of its own, and ours depends on language. It is also asked about everything, so our claims mix fact-checked news with encyclopedic facts, and the gap is as wide on one half as on the other.

Prior literature extensively documents this disparity across cross-lingual verification and dialectal fairness \citep{lucas2026bluff,lucas2026diaharm,qi2023cross,han2026mubench,liu2025entity}. However, existing mitigations either require resource-intensive retraining that is impractical for downstream deployment, or rely on unconstrained inference-time mappings \citep{ghorbanpour2026lsi,wang2025incline} that are free to rescale and project, and so to discard information as well as re-orient it, which makes a change in score hard to read as evidence of what the underlying model knew. Crucially, both approaches overlook a fundamental question: before concluding that a model lacks factual knowledge in a target language, one must first determine whether it possesses the internal knowledge yet fails to express it.

\noindent\textbf{Research question.} \emph{When a model judges a claim worse outside English, how much of the gap is knowledge it lacks, and how much is knowledge it holds but fails to read out? And can the second part be recovered at inference time, without changing a single weight?}

Both parts are real, and of similar size. Linear probing on intermediate activations recovers the truth in every language we test, including those where the model's own verdict falls to chance, and at those layers the probe scores far above the model's own output head. Comparing probe against head within each language splits the gap in two: $0.026$ is a shortfall in what the model encodes, and $0.021$ is the head reading non-English less efficiently than English (Section~\ref{sec:decomposition}). Only the second part can be repaired without changing the pre-trained parameters. We propose \textsc{RoSh}, a per-language shift and rotation applied to the residual stream at selected intermediate layers, computed as the closed-form orthogonal Procrustes solution \citep{schonemann1966} on unlabelled parallel claims and applied at inference, with no gradient taken and no weight modified. It closes $75\%$ of the gap on average and removes it entirely on Qwen-32B and Llama-70B, lifting accuracy from $0.727$ to $0.757$ and cutting the share of English-correct claims lost in translation by a fifth, from $15.7\%$ to $12.3\%$. The largest gains land where the gap was worst: Arabic on Llama-3B recovers from $0.523$ to $0.841$, Portuguese on Mistral-7B from $0.646$ to $0.896$. What remains is no longer a read-out failure: afterwards the head recovers at least as much of what is encoded outside English as it does in English.


The gains carry beyond our dataset: on the two benchmarks of the closest inference-time method, latent-space intervention \citep{ghorbanpour2026lsi}, run with its own data and evaluation scripts, \textsc{RoSh}'s gains are five to thirteen times larger, and neither benchmark contains news claims (Table~\ref{tab:external}).

\paragraph{Contributions.}
\begin{itemize}
\itemsep3pt

\item \textbf{The verdict depends on the language.} English leads by $+0.046$ AUROC on average and by $0.334$ on Arabic for Llama-3B, over $1{,}500$ claims that appear in every language.

\item \textbf{The gap has two parts, and only one is read-out.} A probe splits it into $0.026$ of representational shortfall and $0.021$ of read-out loss.

\item \textbf{\textsc{RoSh} removes the read-out part entirely.} A closed-form, label-free shift and rotation at three layers closes $75\%$ of the gap, and eleven controls, including an unconstrained map that falls below the untouched baseline, rule out generic perturbation.

\end{itemize}

\section{Related Work}
\label{sec:related}

Who receives a good judgement. Unequal judgements predate language models: readers of fact-checked news trusted claims that matched their politics and distrusted the rest \citep{babaei2021analyzing}; we ask the same question with a model as the judge and the input language in place of politics. Translating claims into English does not settle it, since translation-based pipelines carry disparities of their own \citep{singhal2024comparative}. The disparity itself is well measured: BMLAMA and RankC \citep{qi2023cross} and MuBench, across 61 languages \citep{han2026mubench}, test whether a model answers the same fact alike across languages, \citet{liu2025entity} tie consistency to entity alignment and improve it by English subject substitution, and \citet{piratla2025rethinking} trace part of the gap to response variance and reduce it by inference-time ensembling.

Where knowledge sits. Llama-2's intermediate layers favour English equivalents of the intended output \citep{wendler2024llamas}, cross-lingual factual failures arise both in recall and in converting the answer to the target language \citep{lu2025paths}, and factual associations sit in middle-layer MLP modules that can be edited directly \citep{meng2022locating}. Truth in particular can be read from hidden states, without supervision by enforcing consistency across paired answers \citep{burns2022discovering} or with a small supervised classifier \citep{azaria2023internal}; true and false statements separate geometrically, though probe directions of similar accuracy can differ when used to intervene \citep{marks2023geometry}, and a general truth direction can be told apart from a polarity-sensitive one that reverses under negation \citep{burger2024truth}. Our probes fit this picture: the truth is recoverable at mid-depth in every language, and what varies is whether it reaches the output head.

Aligning representations to close the gap. The closest work intervenes on representations at inference. INCLINE fits layer-wise alignment matrices by least squares on parallel sentences \citep{wang2025incline}; latent-space intervention trains layer-wise autoencoders that align languages in a shared latent space, improving cross-lingual consistency while largely preserving accuracy \citep{ghorbanpour2026lsi}; steering vectors shift target-language activations toward English at a single layer \citep{mahmoud2026improving}; and \citet{agarwal2025aligning} instead fine-tune on batches of semantically equivalent multilingual examples.

What is new. Every intervention above fits an unconstrained map or vector, free to rescale and project, and so to discard information as well as re-orient it; a change in score then mixes the two effects. \textsc{RoSh} is an affine isometry, distance-preserving and invertible, so it provably loses nothing (Proposition~\ref{prop:info}). The estimator itself is old: \citet{schonemann1966} solved the orthogonal Procrustes problem in closed form, \citet{smith2017offline} showed that the orthogonal constraint improves bilingual word-embedding alignment, and \citet{conneau2018word} made it unsupervised. What is new is that the same relationship holds inside a running model, between languages, in the residual stream, at inference time.
\section{Method: Rotation--Shift Alignment (RoSh)}
\label{sec:method}

A transformer's verdict on a factual claim is read from a single direction in its residual stream. We observe that the truth is linearly readable in every language we test, while the read-out head recovers much less of it outside English.
Table~\ref{tab:probe} shows it for two models. On Llama-3B in Arabic the model's own verdict is at chance, $0.523$ AUROC, yet a linear probe on the same activations reaches $0.871$; in every language the probe stays within $0.06$ of its English value, while the read-out falls up to $0.334$ behind. Across all eight models the probe reaches $0.910$ outside English where the head reaches $0.823$ (Table~\ref{tab:probeall}), and Section~\ref{sec:decomposition} splits the gap into what the model fails to encode and what it fails to read out. Figure~\ref{fig:projection_strips} shows the same claims from the read-out's own perspective, before and after the correction. RoSh acts on that directly, applying a closed-form orthogonal rotation and a mean shift to the hidden states, per language and per layer, without changing any model weight. Figure~\ref{fig:architecture} gives the picture; Appendix~\ref{app:protocol} shows how the lookup table behind it is built.

\begin{table}[!htbp]
\centering
\caption{The truth is encoded even where the verdict fails: linear probe against the model's own read-out, held out, at the first selected layer (for English, the layer most often selected first). Read-out values are those of Table~\ref{tab:perlanguage}.}
\label{tab:probe}
\small
\setlength{\tabcolsep}{4pt}
\begin{tabular}{l l cccccccc}
\toprule
Model & & EN & PT & PL & DE & AR & ZH & RU & TR \\
\midrule
Llama-3B & probe & 0.922 & 0.908 & 0.906 & 0.908 & 0.871 & 0.875 & 0.888 & 0.884 \\
         & read-out & 0.857 & 0.787 & 0.794 & 0.851 & 0.523 & 0.772 & 0.787 & 0.669 \\
\midrule
Gemma-9B & probe & 0.935 & 0.928 & 0.936 & 0.937 & 0.912 & 0.899 & 0.927 & 0.920 \\
         & read-out & 0.878 & 0.859 & 0.862 & 0.862 & 0.844 & 0.762 & 0.857 & 0.820 \\
\bottomrule
\end{tabular}
\end{table}


\begin{figure}[!htbp]
  \centering
  \includegraphics[width=\textwidth]{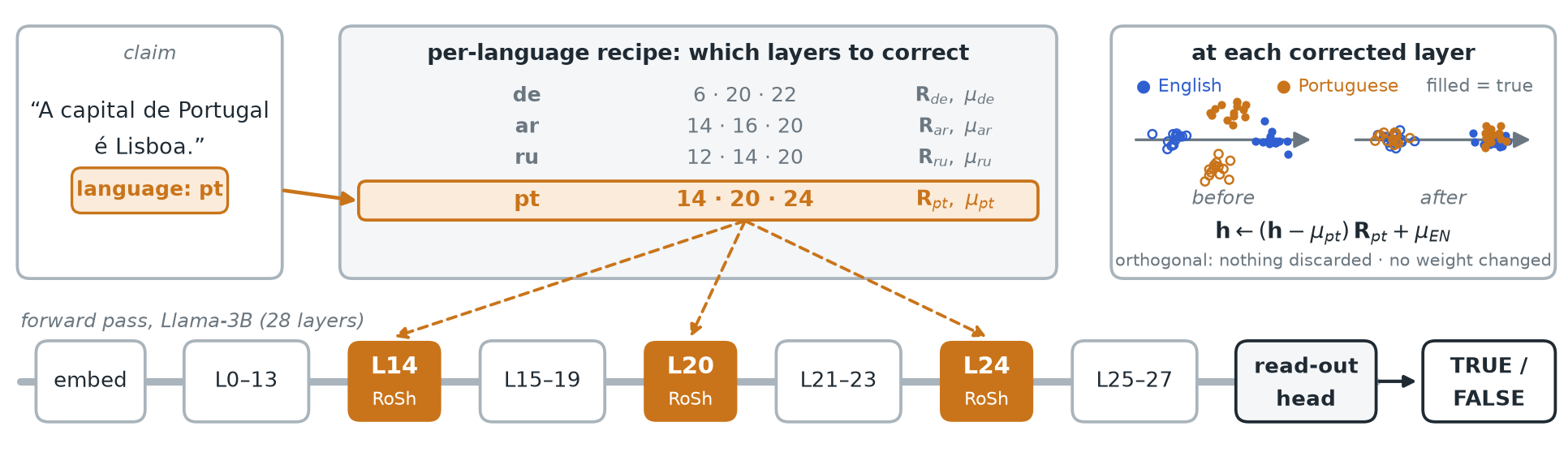}
  \caption{\textbf{Language-conditioned correction at inference.}
The language of the incoming claim selects a recipe: the layers to correct and the rotation $\R_\ell$ and shift to apply at each. Shown is Llama-3B with the Portuguese recipe. Only residual-stream activations at the listed layers are transformed; the model's weights are shared across all languages and never modified. The inset is a schematic of what the correction does at one layer, not a measurement.}
  \label{fig:architecture}
\end{figure}

\subsection{Problem Setup}
\label{sec:setup}


Let $\mathcal{M}$ be a frozen transformer of depth $L$ and residual width $d$. For a claim $x$ in language $\ell$, write $\h_l^\ell(x)\in\mathbb{R}^d$ for the residual-stream activation at layer $l$ and the answer-token position. We are given a parallel corpus $\mathcal{D}=\{(x_i^{\mathrm{en}},x_i^\ell)\}_{i=1}^{n}$ of $n$ claims expressed in both English and a target language.


The hypothesis is that the model internally represents the truth value of $x_i$ in every language, but the subspace carrying this signal in language $\ell$ is rotated relative to the English subspace that the read-out head was implicitly trained to decode.
Formally, there exists an orthogonal matrix $\R_l^\ell \in \mathbb{R}^{d \times d}$ such that,
after the affine map
\begin{equation}
  \label{eq:rosh}
  \boxed{\;\h_l^\ell \;\longleftarrow\;
    \bigl(\h_l^\ell - \bmu_l^\ell\bigr)\,\R_l^\ell
    \;+\; \bmu_l^{\mathrm{en}}\;}
\end{equation}
the transformed activations are approximately aligned with their English counterparts, and the model's own read-out head can now decode the correct factual judgement.

A toy example makes Eq.~\ref{eq:rosh} concrete. Take $d=3$ and a read-out that scores a claim by its first coordinate. In English two true claims sit at $(2,\pm1,0)$ and two false ones at $(-2,\pm1,0)$, so they score $+2$ and $-2$ and separate perfectly. In Arabic the same claims sit at $(4,7,0)$ and $(6,7,0)$ (true) and at $(4,3,0)$ and $(6,3,0)$ (false). Truth is still perfectly separated, but along the second coordinate, which the read-out ignores, so true and false claims both score $4$ or $6$ and the verdict is a coin toss. Eq.~\ref{eq:rosh} subtracts the Arabic mean $(5,5,0)$, rotates by $-90^\circ$ and adds the English mean $(0,0,0)$, taking $(4,7,0)$ to $(-1,2,0)$ and then to $(2,1,0)$, exactly its English counterpart. The other three claims land on theirs, and the scores return to $\pm2$. Figure~\ref{fig:projection_strips} shows the same effect on a larger synthetic cloud, and Appendix~\ref{app:toy} derives $\R$ for this example.

\begin{figure}[!ht]
  \centering
  \includegraphics[width=0.85\textwidth]{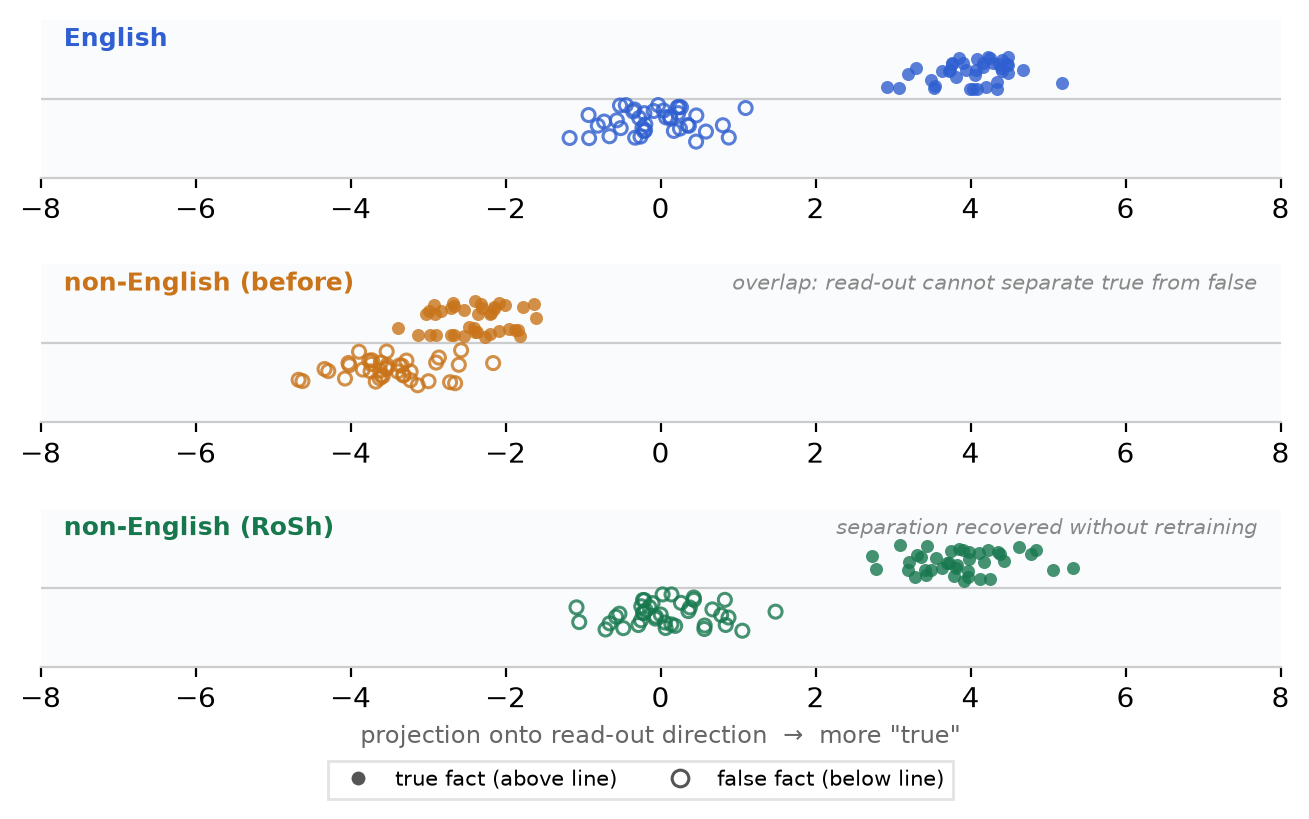}
  \caption{\textbf{What the read-out head sees.} Each dot is one claim projected onto the read-out direction, true above the centre line and false below. English (top) separates cleanly; non-English before correction (middle) overlaps almost completely; after \textsc{RoSh} (bottom) the separation matches English, without retraining. Same synthetic geometry as the inset of Figure~\ref{fig:architecture}.}
  \label{fig:projection_strips}
\end{figure}

\subsection{Fitting the Transform via Orthogonal Procrustes}
\label{sec:procrustes}

Given matched English and target-language activations, we aim to find the rotation, $\R_l^\ell$, that makes the target-language representation look as similar as possible to the English one.
In particular, we propose to recover $\R_l^\ell$ in closed form using the \emph{orthogonal Procrustes problem}~\citep{schonemann1966}.
Given the parallel activations, we first centre each language cloud:
\begin{align}
  \bmu_l^\ell     &= \frac{1}{n}\sum_{i=1}^{n} \h_l^\ell(x_i^\ell),
  &
  \bmu_l^{\mathrm{en}} &= \frac{1}{n}\sum_{i=1}^{n} \h_l^{\mathrm{en}}(x_i^{\mathrm{en}}),
  \label{eq:centroids}
  \\[4pt]
  \mathbf{A}      &= \bigl[\h_l^\ell(x_i^\ell) - \bmu_l^\ell\bigr]_{i=1}^{n}
                     \;\in\mathbb{R}^{n \times d},
  &
  \mathbf{B}      &= \bigl[\h_l^{\mathrm{en}}(x_i^{\mathrm{en}}) - \bmu_l^{\mathrm{en}}\bigr]_{i=1}^{n}
                     \;\in\mathbb{R}^{n \times d}.
  \label{eq:centred}
\end{align}
The rotation that best maps $\mathbf{A}$ onto $\mathbf{B}$ is the one minimising the Frobenius norm of the residual:
\begin{equation}
  \R_l^\ell
    = \argmin_{\R\,:\;\R^\top\!\R = \mathbf{I}}
      \bigl\lVert \mathbf{A}\,\R - \mathbf{B} \bigr\rVert_F.
  \label{eq:procrustes}
\end{equation}
A solution is given by the SVD of the cross-covariance matrix:
\begin{equation}
  \mathbf{M} = \mathbf{A}^\top \mathbf{B}
  \;=\; \mathbf{U}\,\bm{\Sigma}\,\mathbf{V}^\top,
  \qquad
  \R_l^\ell = \mathbf{U}\,\mathbf{V}^\top.
  \label{eq:svd}
\end{equation}

\paragraph{The solution is not unique, and this does not matter.}
With $n = 750$ claims and $d \ge 3{,}072$, $\mathbf{M}$ has rank below $d$, so $\R_l^\ell$ is determined only on the directions the fitting claims span and is free on the rest (Proposition~\ref{prop:family}). Those directions carry at most $1.7\%$ of a held-out activation, and switching to the canonical solution, the optimal $\R$ closest to the identity, moves held-out AUROC by less than $0.001$ (Appendix~\ref{app:nonunique}); the tables use the standard SVD solution.

The full map in Eq.~\ref{eq:rosh} is an affine isometry, distance-preserving and invertible, so it keeps all the information an activation carries about the claim's truth (Proposition~\ref{prop:info}): an improvement cannot come from discarding a distracting part of the representation, only from a change of frame.

\subsection{The Two Components}
\label{sec:decompose}

Equation~\ref{eq:rosh} composes a rotation about the non-English centroid with a translation onto the English one, and each can be run alone:
\begin{equation}
  \text{shift only:}\;\;
  \h_l^\ell \leftarrow \h_l^\ell - \bmu_l^\ell + \bmu_l^{\mathrm{en}},
  \qquad
  \text{rotation only:}\;\;
  \h_l^\ell \leftarrow \bigl(\h_l^\ell - \bmu_l^\ell\bigr)\R_l^\ell + \bmu_l^\ell .
  \label{eq:components}
\end{equation}
A translation cannot change an orientation, so under a linear read-out a pure shift would leave any ranking metric untouched; the corrected state, however, passes through further non-linear blocks, so the shift does contribute: alone it recovers $43\%$ of the full gain against $74\%$ for the rotation (the two are not additive).
\subsection{Greedy Layer Selection}
\label{sec:layer_selection}

Not every layer carries a misalignment worth correcting, and the useful ones are not independent: once a layer is corrected, the activations arriving at every later layer are different, so the best second layer depends on which one was chosen first. Searching all combinations is prohibitive, so we use greedy forward selection on the fit split alone. That is, we greedily build a small sequence of intervention layers, repeatedly choosing the layer that gives the largest additional gain after the previously selected corrections are already in place.

For each language we sweep every $s$-th layer between $15\%$ and $97\%$ of depth, with $s = 2$, or $3$ for the models of 64 to 95 layers, giving 12 to 27 candidates. We start at $15\%$ because early runs showed that the lowest layers do not yet carry enough truth for a correction to help, while later layers do: in the lowest $30\%$ of depth a probe reads the truth outside English at $0.76$ AUROC against about $0.90$ higher up, and only 2 of the 56 best single-layer corrections fall there.
The fit split is halved into two inner folds of 375 claims: at each candidate we fit $\R_l^\ell$ on one and score the corrected model's AUROC on the other. The best layer is selected and its transform refitted on all 750 fitting claims. The sweep is then repeated with that correction in place, up to $D = 3$ layers, each rotation fitted on activations as they leave the previous correction.
Greedy search keeps this affordable: on Mistral-7B it takes $14 + 13 + 12 = 39$ fits, where every three-layer combination would take $\binom{14}{3} = 364$. Ranking by AUROC uses labels on the fitting split; the transform itself does not. The output is a per-language \emph{recipe}: at most three layers, each with its rotation $\R_{l_k}^\ell$ and centroid pair $(\bmu_{l_k}^\ell,\bmu_{l_k}^{\mathrm{en}})$.

\paragraph{Inference.}
At test time no additional data is needed: the language selects its recipe, the forward pass runs as normal with the residual stream at the answer position transformed by Eq.~\ref{eq:rosh} at each listed layer, and the verdict is read from the unmodified output head. The cost is one $d \times d$ multiply at up to three layers, on a single token. Algorithm~\ref{alg:rosh} in Appendix~\ref{app:algorithm} gives the procedure in full.


\subsection{Controls}
\label{sec:controls}

Eleven controls test that the transform does what we claim. Each runs on identical data at the same
layers and breaks exactly one assumption: whether each component is needed, whether the claim-level
pairing matters, whether any perturbation of matched size would serve, and whether the map is
specific to its layer and direction. A control passes when it scores below the full method
(Appendix~\ref{app:controls}).

\section{Data and Experimental Setup}
\label{sec:setup-data}

\paragraph{Claims.}
\label{sec:data}
We evaluate on $1{,}500$ factual claims\footnote{Code and seeds: \url{https://anonymous.4open.science/r/fair-factchecking-40E8}.} with known truth values, balanced between true and
false, each in identical form in eight languages: English, Portuguese, Polish, German,
Arabic, Chinese, Russian and Turkish, spanning four scripts and four families. The set mixes
claims of the kind fact-checkers adjudicate with encyclopedic facts, and the non-English
versions are taken as released from BMLAMA-53 \citep{qi2023cross}; Appendix~\ref{app:prompts}
lists the claim types and gives examples. The claims are split once
into $750$ for fitting and $750$ held out: every transform, layer choice and threshold is
estimated on the fit half, the held-out half is scored exactly once, and a claim falls on
the same side of the split in every language (Appendix~\ref{app:setup}).

\paragraph{Models.}
\label{sec:models}
We use eight instruction-tuned models from five families, from 3B to 70B parameters and with
residual widths from $3{,}072$ to $8{,}192$ (Table~\ref{tab:models}). All are frozen and
none is fine-tuned. They run in half precision (bf16 for the two Gemma models), except
Llama-70B and DeepSeek-67B, which run in 4-bit. Every prompt is round-tripped through its
tokeniser before any activation is collected (Appendix~\ref{app:tokenizer}).

\paragraph{Scoring.}
\label{sec:prompting}
Each claim is presented in its own language with a fixed template that asks for a single
verdict token (Appendix~\ref{app:prompts}). The score of a claim is the logit difference between the \textsc{true} and
\textsc{false} tokens at the answer position,
\begin{equation}
  s(x) = \mathbf{u}^\top \h_L(x),
  \qquad
  \mathbf{u} = \mathbf{w}_{\textsc{true}} - \mathbf{w}_{\textsc{false}},
  \label{eq:score}
\end{equation}
read from the unmodified head after RoSh has changed $\h_l$ at the selected layers.

\paragraph{Metrics.}
\label{sec:metrics}
We report held-out AUROC of $s(x)$ as the primary metric, the language gap (English minus
the mean of the other seven languages) and the share of it closed, accuracy, directional
consistency, and a best-layer linear probe; Appendix~\ref{app:setup} defines each.

\paragraph{External benchmarks.}
To compare with latent-space intervention \citep{ghorbanpour2026lsi} on its own ground, we
also run \textsc{RoSh} on its two benchmarks, KLAR and mParaRel, with three 8B models
(Qwen3-8B, Llama-3.1-8B and Aya-8B), using their data and metric code
(Appendix~\ref{app:external}).

\section{Results}
\label{sec:results}

Every transform is fitted on 750 claims and scored on the 750 it never saw. Accuracy and
consistency are read at each language's median operating point, fitted on the training half.

\begin{table}[ht]
\centering
\caption{Non-English mean, held out. Gap is English minus non-English, and gap closed is
$(\mathrm{gap}_{\text{before}} - \mathrm{gap}_{\text{after}})/\mathrm{gap}_{\text{before}}$
(Appendix~\ref{app:setup}). Directional consistency is the share of claims the model gets
right in English that it also gets right in the other language. The Mean row averages the
eight model rows; its gap closed is the mean of the per-model values, with values above
100\% counted as 100\%.}
\label{tab:main}
\small
\setlength{\tabcolsep}{4pt}
\resizebox{\textwidth}{!}{%
\begin{tabular}{l ccc cc c cc}
\toprule
& & \multicolumn{2}{c}{non-EN} & \multicolumn{2}{c}{gap} & gap & & directional \\
\cmidrule(lr){3-4}\cmidrule(lr){5-6}
Model & EN & before & after & before & after & closed & accuracy & consistency \\
\midrule
Llama-3B         & 0.857 & 0.740 & 0.829 & $+0.116$ & $+0.028$ & 76\%   & $0.671 \to 0.736$ & $0.750 \to 0.841$ \\
Mistral-7B       & 0.834 & 0.704 & 0.809 & $+0.130$ & $+0.025$ & 80\%   & $0.649 \to 0.726$ & $0.715 \to 0.806$ \\
Gemma-9B         & 0.878 & 0.838 & 0.865 & $+0.040$ & $+0.012$ & 69\%   & $0.740 \to 0.759$ & $0.879 \to 0.905$ \\
Mistral-24B      & 0.888 & 0.866 & 0.878 & $+0.022$ & $+0.011$ & 53\%   & $0.756 \to 0.767$ & $0.883 \to 0.898$ \\
Gemma-27B        & 0.860 & 0.833 & 0.844 & $+0.027$ & $+0.016$ & 42\%   & $0.729 \to 0.746$ & $0.839 \to 0.851$ \\
Qwen-32B         & 0.885 & 0.879 & 0.889 & $+0.006$ & $-0.004$ & $>$100\% & $0.770 \to 0.778$ & $0.903 \to 0.911$ \\
Llama-70B        & 0.879 & 0.865 & 0.883 & $+0.014$ & $-0.003$ & $>$100\% & $0.752 \to 0.780$ & $0.895 \to 0.915$ \\
DeepSeek-67B     & 0.875 & 0.861 & 0.872 & $+0.014$ & $+0.003$ & 80\%   & $0.748 \to 0.762$ & $0.882 \to 0.892$ \\
\midrule
Mean             & 0.870 & 0.823 & 0.859 & $+0.046$ & $+0.011$ & 75\%   & $0.727 \to 0.757$ & $0.843 \to 0.877$ \\
\bottomrule
\end{tabular}}
\end{table}

The gap closes on every model, by 75\% on average, and Qwen-32B and Llama-70B end with
non-English slightly ahead of English. Where the gap was widest the correction is largest:
Llama-3B and Mistral-7B start $0.116$ and $0.130$ behind and end within $0.03$. Directional consistency rises from $0.843$ to $0.877$ (of the claims a model gets right in English, the share it gets wrong in translation falls from $15.7\%$ to $12.3\%$), so the
claims English handles correctly survive translation more often than before. Every tokeniser was verified in all eight
languages before any run (Appendix~\ref{app:tokenizer}).

\begin{table}[ht]
\centering
\caption{Baseline and corrected AUROC per language, held out. Shading shows the change: green for a gain, darker the larger it is; red for a drop; grey for no change.}
\label{tab:perlanguage}
\scriptsize
\setlength{\tabcolsep}{3pt}
\resizebox{\textwidth}{!}{%
\begin{tabular}{l ccccccc}
\toprule
Model & PT & PL & DE & AR & ZH & RU & TR \\
\midrule
Llama-3B & \cellcolor{green!36}$0.787 \to 0.840$ & \cellcolor{green!36}$0.794 \to 0.829$ & \cellcolor{green!10}$0.851 \to 0.856$ & \cellcolor{green!52}$0.523 \to 0.841$ & \cellcolor{green!22}$0.772 \to 0.798$ & \cellcolor{green!36}$0.787 \to 0.838$ & \cellcolor{green!52}$0.669 \to 0.798$ \\
Mistral-7B & \cellcolor{green!52}$0.646 \to 0.896$ & \cellcolor{green!36}$0.802 \to 0.878$ & \cellcolor{green!36}$0.807 \to 0.847$ & \cellcolor{green!36}$0.585 \to 0.674$ & \cellcolor{green!36}$0.659 \to 0.702$ & \cellcolor{green!36}$0.791 \to 0.869$ & \cellcolor{green!52}$0.639 \to 0.795$ \\
Gemma-9B & \cellcolor{green!22}$0.859 \to 0.872$ & \cellcolor{green!22}$0.862 \to 0.884$ & \cellcolor{green!10}$0.862 \to 0.871$ & \cellcolor{green!10}$0.844 \to 0.853$ & \cellcolor{green!36}$0.762 \to 0.845$ & \cellcolor{green!22}$0.857 \to 0.868$ & \cellcolor{green!36}$0.820 \to 0.865$ \\
Mistral-24B & \cellcolor{gray!15}$0.874 \to 0.874$ & \cellcolor{green!22}$0.875 \to 0.890$ & \cellcolor{green!10}$0.893 \to 0.897$ & \cellcolor{green!36}$0.859 \to 0.892$ & \cellcolor{green!10}$0.824 \to 0.828$ & \cellcolor{green!22}$0.889 \to 0.902$ & \cellcolor{green!22}$0.846 \to 0.860$ \\
Gemma-27B & \cellcolor{red!22}$0.856 \to 0.855$ & \cellcolor{green!22}$0.840 \to 0.860$ & \cellcolor{green!22}$0.847 \to 0.863$ & \cellcolor{green!22}$0.826 \to 0.841$ & \cellcolor{red!22}$0.790 \to 0.787$ & \cellcolor{green!36}$0.830 \to 0.863$ & \cellcolor{red!22}$0.842 \to 0.841$ \\
Qwen-32B & \cellcolor{green!10}$0.879 \to 0.882$ & \cellcolor{green!22}$0.888 \to 0.898$ & \cellcolor{green!10}$0.891 \to 0.900$ & \cellcolor{green!22}$0.877 \to 0.893$ & \cellcolor{green!22}$0.862 \to 0.875$ & \cellcolor{green!22}$0.892 \to 0.908$ & \cellcolor{green!10}$0.866 \to 0.868$ \\
Llama-70B & \cellcolor{green!22}$0.882 \to 0.892$ & \cellcolor{green!22}$0.884 \to 0.907$ & \cellcolor{green!22}$0.885 \to 0.896$ & \cellcolor{green!22}$0.851 \to 0.862$ & \cellcolor{green!36}$0.820 \to 0.854$ & \cellcolor{green!36}$0.872 \to 0.904$ & \cellcolor{green!10}$0.861 \to 0.866$ \\
DeepSeek-67B & \cellcolor{green!10}$0.865 \to 0.870$ & \cellcolor{green!22}$0.880 \to 0.896$ & \cellcolor{green!22}$0.871 \to 0.888$ & \cellcolor{green!22}$0.642 \to 0.658$ & \cellcolor{green!22}$0.855 \to 0.865$ & \cellcolor{green!22}$0.876 \to 0.886$ & \cellcolor{red!22}$0.820 \to 0.815$ \\
\bottomrule
\end{tabular}}
\end{table}

Arabic on Llama-3B moves from $0.523$, which is a coin toss, to $0.841$, and Portuguese on
Mistral-7B from $0.646$ to $0.896$ (Table~\ref{tab:perlanguage}). Of the 56 model-language cells, 51 improve, one is
unchanged and the largest regression is $0.005$.
Most languages that stay flat were already
close to English, so there was little left to close; Gemma-27B Chinese ($0.790$ to $0.787$)
is the exception.

\paragraph{Beyond our data.} On the two benchmarks of the closest inference-time method,
latent-space intervention (LSI), run with its own data and metric code on three 8B models,
\textsc{RoSh} raises KLAR accuracy by $5.63$ points and agreement with English by $5.48$,
against the $0.44$ and $1.02$ that \citet{ghorbanpour2026lsi} report for LSI, and raises
mParaRel agreement on Aya by $2.24$ against $0.27$ (Table~\ref{tab:external}). Neither
benchmark contains news claims, and on mParaRel no translation system its templates come
from shows a significant loss (Appendix~\ref{app:mt}), so the correction is not specific to
our data.

\subsection{What \textsc{RoSh} fixes, and what remains}
\label{sec:decomposition}

Comparing the model's read-out with a probe on the same activations splits the gap in two
(Appendix~\ref{app:angle}, Table~\ref{tab:probeall}). The non-English representation is
genuinely weaker, since a probe reaches $0.910$ against $0.935$ in English, a shortfall of
$0.026$. On top of that the head reads non-English less efficiently, losing $0.086$ AUROC
relative to the probe outside English against $0.066$ within it, a further $0.021$. The two
account for the $+0.046$ we measure.

\textsc{RoSh} acts on the second component and removes it. After correction the head loses
$0.051$ relative to the probe in non-English, below the $0.066$ it loses in English on
English's own representation, and it is at or below the English level on every model except
Mistral-24B, where it exceeds it by $0.002$. So the residual gap of $0.011$ is smaller than
the representational shortfall of $0.026$: the corrected head now reads non-English
\emph{more} efficiently than English, by $0.015$, which offsets part of what the
representation lacks.

Our analysis also turns up a finding that has nothing to do with language: even in English
the head leaves $0.066$ AUROC on the table relative to a probe on its own intermediate
activations. Read-out inefficiency therefore appears to be a general property of the model,
while non-English languages suffer an additional read-out penalty that \textsc{RoSh}
specifically removes.

\subsection{Selected layers}
\label{sec:layersselected}

The search settles on three layers for every model, mostly in the upper middle of the network
but often reaching early layers, and the stacks differ by language within a model: on Llama-3B,
Portuguese takes 14, 20 and 24 and Arabic 16, 20 and 14. The recipe is therefore fitted per
language (Appendix~\ref{app:layers}).

\subsection{Controls}
\label{sec:results_controls}

A gain on held-out data does not establish that the transform does what we claim, so we ran the eleven arms of Section~\ref{sec:controls}, each fitted and applied exactly as the method is. Averaged over eight models and seven languages, \textsc{RoSh} reaches $0.853$ against an untouched baseline of $0.820$. The eight arms scored on non-English all land below $0.853$, and the three scored on English land below the English baseline of $0.870$, which is the comparison they call for. The informative ones are close together in construction and far apart in score. Scrambling the item pairing before fitting, which preserves both clouds and destroys only which claim matches which, falls to $0.621$; a random rotation of matched size falls to $0.658$; an unconstrained linear map on the
same pairs reaches $0.786$, below doing nothing at all. Fitting English against itself returns $0.870$ and moves the score by $0.000$, which is the null the isometry argument requires. Per model the fitted map beats all three scrambled runs, with $z$ between $+1.83$ and
$+8.13$. An earlier six-model version of this study used a single language-direction steering vector, which failed this same control on four of the six, and was abandoned for that reason. Tables~\ref{tab:scrambled} and \ref{tab:arms} in Appendix~\ref{app:controlresults} give every arm.
\FloatBarrier

\section{Discussion}
\label{sec:discussion}

\paragraph{The gap is partly a read-out failure.}
A linear probe separates true from false in every language, including where the model's own verdict is near chance, so the content survives in the residual stream. Of the $0.046$ gap, $0.021$ is the head failing to reach that content, which \textsc{RoSh} removes, and $0.026$ is a weaker representation, which no isometry can repair. Because the transform is invertible (Proposition~\ref{prop:info}), an improvement can only mean the head was failing to extract something already present.

\paragraph{The constraint carries the result.}
The same pairs fitted without the orthogonality constraint give $0.786$, below the untouched baseline of $0.820$, against $0.853$ for the constrained solution: with 750 pairs in thousands of dimensions the least-squares map is ill-conditioned and distorts held-out activations (Appendix~\ref{app:ridge}). Orthogonality removes that failure by construction and is what licenses the reading above.

\paragraph{Relation to weight-space methods.}
We claim no impossibility result: fine-tuning could in principle reach a model that behaves similarly. What we claim is that the correction can be \emph{found} in closed form, from unlabelled parallel claims, in one SVD, with no gradient step and no labelled data in the target language, which is precisely the resource the affected communities lack. Nor is it a weight edit: it acts only at the answer position, differs per language and includes a shift, so it cannot be folded into the checkpoint. Two guarantees follow that have no counterpart in weight space. English is left exactly unchanged (Proposition~\ref{prop:english}; the null is $0.000$ in Table~\ref{tab:arms}), and since only one language's recipe is active in any forward pass, improving one language cannot disturb another.


\paragraph{What this means for deployment.}
The transform is a per-language matrix and a per-language vector at three layers,
estimated from 750 parallel claims, whose truth labels are used only to choose the layers,
and applied at one token
position during inference. It requires no gradient, no labelled target-language data, and no
change to the shipped checkpoint, which places it within reach of whoever deploys a model
rather than only whoever trained it. That matters for the specific unfairness we measure,
since the languages with the largest gaps are also those with the least training data, the
fewest labelled resources and the weakest locally available systems.

\subsection{Limitations}
\label{sec:limitations}

Eight languages and four scripts show the effect without characterising it across the long
tail, where the disparity is likely worst, and the non-English claims, taken from BMLAMA-53
\citep{qi2023cross}, were translated rather than written natively. On mParaRel, the second
benchmark of \citet{ghorbanpour2026lsi}, whose templates come from five translation systems,
the correction survives a split by system (Appendix~\ref{app:mt}). All results use one verdict-token template, so we cannot speak to
chain-of-thought or open-ended answers. The method needs claims paired across languages and
does not apply where no parallel corpus exists. Truth labels do not fit the transform but do
enter through layer selection, so the label-free claim covers the transform rather than the
whole recipe. And the diagnosis rests on probes: we do not identify which
part of the head fails, or what puts non-English representations in a different frame.

\section{Conclusion}
\label{sec:conclusion}

Ask a model to judge the same claim in English and then in another language and the two
verdicts often disagree. Across eight models from five families and $1{,}500$ claims rendered
identically in eight languages, English leads by $+0.046$ AUROC on average and by $0.334$ on
Arabic for Llama-3B, where the model drops to $0.523$ and is guessing. Every claim appears in
every language, so this is not a matter of some questions being harder.

Part of what the model needs is already there. A linear probe recovers the truth from
mid-network activations even where the model's own verdict sits at chance, and separating the
two shows that nearly half of the gap is the read-out failing to reach what is encoded,
the rest being a genuinely weaker representation. \textsc{RoSh} targets the first part, with a
per-language shift and rotation of the residual stream at three layers, computed in closed
form from unlabelled parallel claims, no gradient taken and no weight modified. It closes
$75\%$ of the gap on average and removes it on two models, lifts accuracy from $0.727$ to
$0.757$, and cuts the share of English-correct claims lost in translation by a fifth, from
$15.7\%$ to $12.3\%$. The same recipe transfers to two benchmarks built by other authors, where
its gains are five to thirteen times those of the closest inference-time method, on that
method's own data and metric code.

An unconstrained map on the same pairs falls below doing nothing at all. What is left is to reach the long tail of languages, where the disparity is likely
largest and parallel data scarcest, and to find what puts non-English representations in a
different frame in the first place.


\subsection*{AI use statement}
In this work, we used generative AI tools (a large language model assistant) for editing
and condensing the manuscript text, drafting and checking LaTeX, writing experiment, analysis
and plotting code, and cross-checking the numbers reported in the text against our result
files. We have not used generative AI tools to produce any experimental result: every number
in the paper is computed by our code from the models' outputs, and every AI-drafted script
was run and checked by the authors.
We have reviewed all AI-assisted work and take responsibility for the final content of this
work, including text, claims or artifacts produced with the aid of generative AI.


\bibliography{references}

@techreport{reuters2026,
  title       = {Digital News Report 2026},
  author      = {Newman, Nic and Fletcher, Richard and Robertson, Craig T.
                 and Arguedas, Amy Ross and Nielsen, Rasmus Kleis},
  institution = {Reuters Institute for the Study of Journalism,
                 University of Oxford},
  year        = {2026},
  note        = {Survey of approximately 2,000 respondents in each of 48 markets},
  url         = {https://reutersinstitute.politics.ox.ac.uk/digital-news-report/2026}
}

@inproceedings{lucas2026bluff,
  title     = {{BLUFF}: Benchmarking the Detection of False and Synthetic
               Content across 58 Low-Resource Languages},
  author    = {Lucas, Jason and Murtagh-White, Matt and Uchendu, Adaku and
               Al-Lawati, Ali and Yamashita, Michiharu and Macko, Dominik and
               Srba, Ivan and Moro, Robert and Lee, Dongwon},
  booktitle = {Proceedings of the 32nd ACM SIGKDD International Conference on
               Knowledge Discovery and Data Mining (KDD)},
  year      = {2026},
  address   = {Jeju, Korea}
}

@inproceedings{lucas2026diaharm,
  title     = {{DIA-HARM}: Dialectal Disparities in Harmful Content Detection
               Across 50 English Dialects},
  author    = {Lucas, Jason and Murtagh-White, Matt and Al-Lawati, Ali and
               Uchendu, Uchendu and Uchendu, Adaku and Lee, Dongwon},
  booktitle = {Proceedings of the 64th Annual Meeting of the Association for
               Computational Linguistics (ACL)},
  year      = {2026},
  address   = {San Diego, CA}
}

@article{wendler2024llamas,
  title   = {Do Llamas Work in English? On the Latent Language of
             Multilingual Transformers},
  author  = {Wendler, Chris and Veselovsky, Veniamin and Monea, Giovanni and
             West, Robert},
  journal = {arXiv preprint arXiv:2402.10588},
  year    = {2024}
}

@inproceedings{azaria2023internal,
  title     = {The Internal State of an {LLM} Knows When It's Lying},
  author    = {Azaria, Amos and Mitchell, Tom},
  booktitle = {Findings of the Association for Computational Linguistics:
               EMNLP 2023},
  year      = {2023}
}

@article{marks2023geometry,
  title   = {The Geometry of Truth: Emergent Linear Structure in Large
             Language Model Representations of True/False Datasets},
  author  = {Marks, Samuel and Tegmark, Max},
  journal = {arXiv preprint arXiv:2310.06824},
  year    = {2023}
}

@inproceedings{burger2024truth,
  title     = {Truth is Universal: Robust Detection of Lies in {LLM}s},
  author    = {B{\"u}rger, Lennart and others},
  booktitle = {Advances in Neural Information Processing Systems (NeurIPS)},
  year      = {2024}
}

@inproceedings{wang2025incline,
    title     = "Bridging the Language Gaps in Large Language Models with
                 Inference-Time Cross-Lingual Intervention",
    author    = "Wang, Weixuan and Wu, Minghao and Haddow, Barry and
                 Birch, Alexandra",
    booktitle = "Proceedings of the 63rd Annual Meeting of the Association for
                 Computational Linguistics (Volume 1: Long Papers)",
    month     = jul,
    year      = "2025",
    address   = "Vienna, Austria",
    publisher = "Association for Computational Linguistics",
    url       = "https://aclanthology.org/2025.acl-long.270/",
    doi       = "10.18653/v1/2025.acl-long.270",
    pages     = "5418--5433"
}

@article{schonemann1966,
  title   = {A Generalized Solution of the Orthogonal {P}rocrustes Problem},
  author  = {Sch{\"o}nemann, Peter H.},
  journal = {Psychometrika},
  volume  = {31},
  number  = {1},
  pages   = {1--10},
  year    = {1966},
  doi     = {10.1007/BF02289451}
}

@inproceedings{smith2017offline,
  title     = {Offline Bilingual Word Vectors, Orthogonal Transformations and
               the Inverted Softmax},
  author    = {Smith, Samuel L. and Turban, David H. P. and Hamblin, Steven and
               Hammerla, Nils Y.},
  booktitle = {International Conference on Learning Representations (ICLR)},
  year      = {2017}
}

@inproceedings{conneau2018word,
  title     = {Word Translation Without Parallel Data},
  author    = {Conneau, Alexis and Lample, Guillaume and Ranzato, Marc'Aurelio
               and Denoyer, Ludovic and J{\'e}gou, Herv{\'e}},
  booktitle = {International Conference on Learning Representations (ICLR)},
  year      = {2018}
}

@article{allcott2017social,
  title   = {Social Media and Fake News in the 2016 Election},
  author  = {Allcott, Hunt and Gentzkow, Matthew},
  journal = {Journal of Economic Perspectives},
  volume  = {31},
  number  = {2},
  pages   = {211--236},
  year    = {2017},
  doi     = {10.1257/jep.31.2.211}
}

@article{vosoughi2018spread,
  title   = {The spread of true and false news online},
  author  = {Vosoughi, Soroush and Roy, Deb and Aral, Sinan},
  journal = {Science},
  volume  = {359},
  number  = {6380},
  pages   = {1146--1151},
  year    = {2018},
  doi     = {10.1126/science.aap9559}
}

@article{babaei2021analyzing,
  title={Analyzing biases in perception of truth in news stories and their implications for fact checking},
  author={Babaei, Mahmoudreza and Kulshrestha, Juhi and Chakraborty, Abhijnan and Redmiles, Elissa M and Cha, Meeyoung and Gummadi, Krishna P},
  journal={IEEE Transactions on Computational Social Systems},
  volume={9},
  number={3},
  pages={839--850},
  year={2021},
  publisher={IEEE}
}

@article{ghorbanpour2026lsi,
  title   = {Latent-Space Intervention for Cross-Lingual Factual Consistency:
             Consistency Improvements without Accuracy Drops},
  author  = {Ghorbanpour, Faeze and Fierro, Constanza and Fraser, Alexander
             and S{\o}gaard, Anders},
  journal = {arXiv preprint arXiv:2608.28860},
  year    = {2026},
  note    = {To appear at EMNLP 2026},
  doi     = {10.48550/arXiv.2608.28860},
  url     = {https://arxiv.org/abs/2608.28860}
}

@article{singhal2024comparative,
  title={A comparative study of translation bias and accuracy in multilingual large language models for cross-language claim verification},
  author={Singhal, Aryan and Shao, Veronica and Sun, Gary and Ding, Ryan and Lu, Jonathan and Zhu, Kevin},
  journal={arXiv preprint arXiv:2410.10303},
  year={2024}
}

@inproceedings{qi2023cross,
  title={Cross-lingual consistency of factual knowledge in multilingual language models},
  author={Qi, Jirui and Fern{\'a}ndez, Raquel and Bisazza, Arianna},
  booktitle={Proceedings of the 2023 Conference on Empirical Methods in Natural Language Processing},
  pages={10650--10666},
  year={2023}
}

@inproceedings{han2026mubench,
  title={Mubench: Assessment of multilingual capabilities of large language models across 61 languages},
  author={Han, Wenhan and Zhang, Yifan and Chen, Zhixun and Pechenizkiy, Mykola and Fang, Meng and Zheng, Yin and others},
  booktitle={Findings of the Association for Computational Linguistics: ACL 2026},
  pages={16163--16192},
  year={2026}
}

@article{liu2025entity,
  title={On the entity-level alignment in crosslingual consistency},
  author={Liu, Yihong and Wang, Mingyang and Yvon, Fran{\c{c}}ois and Sch{\"u}tze, Hinrich},
  journal={arXiv preprint arXiv:2510.10280},
  year={2025}
}

@article{piratla2025rethinking,
  title={Rethinking Cross-lingual Gaps from a Statistical Viewpoint},
  author={Piratla, Vihari and Jain, Purvam and Singh, Darshan and Cohn, Trevor and Jyothi, Preethi and Talukdar, Partha},
  journal={arXiv preprint arXiv:2510.15551},
  year={2025}
}

@inproceedings{lu2025paths,
  title={Paths not taken: Understanding and mending the multilingual factual recall pipeline},
  author={Lu, Meng and Zhang, Ruochen and Eickhoff, Carsten and Pavlick, Ellie},
  booktitle={Proceedings of the 2025 Conference on Empirical Methods in Natural Language Processing},
  pages={15077--15107},
  year={2025}
}

@article{meng2022locating,
  title={Locating and editing factual associations in gpt},
  author={Meng, Kevin and Bau, David and Andonian, Alex and Belinkov, Yonatan},
  journal={Advances in neural information processing systems},
  volume={35},
  pages={17359--17372},
  year={2022}
}

@inproceedings{burns2022discovering,
  title     = {Discovering Latent Knowledge in Language Models Without Supervision},
  author    = {Burns, Collin and Ye, Haotian and Klein, Dan and Steinhardt, Jacob},
  booktitle = {International Conference on Learning Representations (ICLR)},
  year      = {2023}
}

@inproceedings{mahmoud2026improving,
  title={Improving multilingual language models by aligning representations through steering},
  author={Mahmoud, Omar},
  booktitle={Proceedings of the Fifteenth Language Resources and Evaluation Conference (LREC 2026)},
  pages={2090--2103},
  year={2026}
}

@inproceedings{agarwal2025aligning,
  title={Aligning llms for multilingual consistency in enterprise applications},
  author={Agarwal, Amit and Meghwani, Hansa and Patel, Hitesh Laxmichand and Sheng, Tao and Ravi, Sujith and Roth, Dan},
  booktitle={Proceedings of the 2025 Conference on Empirical Methods in Natural Language Processing: Industry Track},
  pages={117--137},
  year={2025}
}
\bibliographystyle{iclr2027_conference}

\appendix

\section{Experimental Setup Details}
\label{app:setup}

\paragraph{Claims.}
We evaluate on $1{,}500$ factual claims with known truth values, balanced between true and
false. The set deliberately mixes two kinds of statement. The first are claims of the sort
that circulate publicly and are adjudicated by fact-checking organisations; the second are
encyclopedic facts of the kind used in cross-lingual knowledge probes. The claims are drawn from BMLAMA-53 \citep{qi2023cross}, released under the
Apache 2.0 licence; a false claim pairs a subject and relation with a distractor object in
place of the true one, and the set is balanced at 750 true and 750 false. Mixing the two lets us test whether the
language gap is a property of contested public claims or of factual verification in
general; Section~\ref{sec:results} reports it separately for each half.

Each claim exists in identical form in eight languages: English, Portuguese, Polish,
German, Arabic, Chinese, Russian and Turkish. The eight span four scripts and four language
families, and cover a wide range of representation in typical pretraining corpora, which is
the axis the gap is expected to follow.

\paragraph{Translation.}
The non-English versions are taken as released from BMLAMA-53 \citep{qi2023cross}, whose
language versions were built by translating English prompt--object pairs from earlier
multilingual probing sets, with entity names taken from Wikidata labels in each language; we
did not translate or review them ourselves. Because translation quality is a possible
confounder, Appendix~\ref{app:mt} checks the correction on mParaRel, whose templates were
produced by five machine-translation systems, separately for each system and restricted to
claims the model answers correctly in English so that item difficulty is held fixed.

\paragraph{Fit and held-out split.}
The $1{,}500$ claims are split once into $750$ for fitting and $750$ held out. Every
transform, every layer-selection decision and every operating threshold is estimated on the
fit half alone. The held-out half is scored exactly once, and all numbers reported in
Section~\ref{sec:results} come from it. Because a claim appears in all eight languages, it
falls on the same side of the split in every language, so no language is ever evaluated on
a claim another language was fitted on.

\paragraph{Models.}
We evaluate eight instruction-tuned models from five families, spanning 3B to 70B parameters
and residual widths from $3{,}072$ to $8{,}192$ (Table~\ref{tab:models}). All are used
frozen, in half precision (bf16 for the two Gemma models; Llama-70B and DeepSeek-67B in
4-bit), with a single forward pass per claim and no decoding. No
model is fine-tuned at any point.

\begin{table}[ht]
\centering
\caption{Models evaluated. Depth is the number of transformer blocks and $d$ the residual
width; both determine the size of the layer sweep and of the fitted transform.}
\label{tab:models}
\small
\begin{tabular}{l l c c}
\toprule
Model & Checkpoint & depth $L$ & width $d$ \\
\midrule
Llama-3B         & \texttt{unsloth/Llama-3.2-3B-Instruct} & 28 & $3{,}072$ \\
Mistral-7B       & \texttt{unsloth/mistral-7b-instruct-v0.3} & 32 & $4{,}096$ \\
Gemma-9B         & \texttt{unsloth/gemma-2-9b-it} & 42 & $3{,}584$ \\
Mistral-24B      & \texttt{unsloth/Mistral-Small-24B-Instruct-2501} & 40 & $5{,}120$ \\
Gemma-27B        & \texttt{unsloth/gemma-2-27b-it} & 46 & $4{,}608$ \\
Qwen-32B         & \texttt{Qwen/Qwen2.5-32B-Instruct} & 64 & $5{,}120$ \\
Llama-70B        & \texttt{unsloth/Llama-3.3-70B-Instruct-bnb-4bit} & 80 & $8{,}192$ \\
DeepSeek-67B     & \texttt{deepseek-ai/deepseek-llm-67b-chat} & 95 & $8{,}192$ \\
\bottomrule
\end{tabular}
\end{table}

\paragraph{Tokeniser verification.}
One model initially resolved to a tokeniser that silently discarded non-Latin script, so
that Chinese claims tokenised to zero tokens and every Arabic, Chinese and Russian prompt
was identical and empty (Appendix~\ref{app:tokenizer}). We therefore round-trip every prompt
in every language through encode and decode for every model, and require exact recovery
before any activation is collected. All eight models pass this check with the checkpoints
listed above.

\paragraph{Prompting and scoring.}
Each claim is presented in its own language with a fixed template that asks for a single
verdict token:

\begin{quote}\small\ttfamily\raggedright
Is the following statement true or false? Answer with only the word TRUE or FALSE.\\ Statement: \{claim\}
\end{quote}

\noindent
and its translation into each target language. The instruction is written in the claim's
own language, while the answer words stay TRUE and FALSE in every language; the prompt is
wrapped in each model's chat template. The score for a claim is the logit
difference between the \textsc{true} and \textsc{false} continuation tokens at the answer
position (Eq.~\ref{eq:score}), where $\mathbf{w}$ are rows of the unembedding matrix and each
of \textsc{true} and \textsc{false} takes the highest-scoring of its capitalisation and
leading-space variants. This is the quantity RoSh acts on
indirectly: the transform changes $\h_l$ at selected layers, and $s$ is read from the
unmodified head.

\paragraph{AUROC.}
Our primary metric is the area under the ROC curve of $s(x)$ against the true label. We use
a ranking metric rather than accuracy because several models are strongly biased toward one
verdict in some languages, which depresses accuracy for a reason unrelated to whether the
model can separate true from false. AUROC measures separability independently of where the
decision threshold sits.

\paragraph{Language gap.}
The gap for a model is the mean AUROC in English minus the mean over the seven other
languages. \emph{Gap closed} is the reduction in that quantity after correction, as a
fraction of the gap before:
\begin{equation*}
  \mathrm{gap} = \mathrm{AUROC}_{\mathrm{en}} - \frac{1}{7}\sum_{\ell \neq \mathrm{en}} \mathrm{AUROC}_{\ell},
  \qquad
  \text{closed} = \frac{\mathrm{gap}_{\text{before}} - \mathrm{gap}_{\text{after}}}{\mathrm{gap}_{\text{before}}},
\end{equation*}
computed per model. The correction leaves English untouched, so only the non-English term
changes. A value above $100\%$ means the corrected non-English mean overtakes English; when
averaging over models, such values count as $100\%$.

\paragraph{Accuracy.}
Accuracy is reported at each language's median operating point, with the threshold fitted
on the training half and applied unchanged to the held-out half.

\paragraph{Directional consistency.}
Consistency is $1 - P(\text{other language wrong} \mid \text{English right})$: among the
claims the model gets right in English, the fraction that survive translation. This is
directional by design. A model that is uniformly wrong in two languages agrees with itself,
and a symmetric agreement score would reward it; ours does not.

\paragraph{Probe.}
To test whether the information is present before the read-out, we fit an L2-regularised
logistic probe on the activation at each candidate layer, with the regularisation strength
chosen by cross-validation on the fit half, and score it once on the held-out half, per
language; we report the best candidate layer. The probe bounds what any linear reader can
extract at that depth. The part of the difference between the probe and the model's own
read-out that exceeds English's own difference is the quantity our method targets
(Section~\ref{sec:decomposition}, Appendix~\ref{app:angle}).

\FloatBarrier

\section{Prompts and Claims}
\label{app:prompts}

\paragraph{Prompt.}
Each claim is sent as a single user message, wrapped in the model's own chat template with the
generation prompt appended, so the next token is the start of the model's answer. The message
is an instruction in the claim's language followed by the claim:

\begin{table}[ht]
\centering
\caption{The instruction in each Latin-script language, verbatim. The Arabic, Chinese and
Russian versions are equivalent translations and are given verbatim in the code release
(\texttt{INSTR} in \texttt{run\_lookup.py}). In every language the answer words stay
\textsc{true} and \textsc{false}.}
\label{tab:prompts}
\small
\begin{tabular}{l p{0.84\textwidth}}
\toprule
Language & Message \\
\midrule
English & Is the following statement true or false? Answer with only the word TRUE or FALSE.\newline Statement: \{claim\} \\
Portuguese & A seguinte afirmação é verdadeira ou falsa? Responda apenas com a palavra TRUE ou FALSE.\newline Afirmação: \{claim\} \\
Polish & Czy poniższe stwierdzenie jest prawdziwe czy fałszywe? Odpowiedz tylko słowem TRUE albo FALSE.\newline Stwierdzenie: \{claim\} \\
German & Ist die folgende Aussage wahr oder falsch? Antworte nur mit dem Wort TRUE oder FALSE.\newline Aussage: \{claim\} \\
Turkish & Aşağıdaki ifade doğru mu yanlış mı? Sadece TRUE veya FALSE kelimesiyle cevap ver.\newline İfade: \{claim\} \\
\bottomrule
\end{tabular}
\end{table}

\noindent
No answer is generated and none is forced. After one forward pass we read the logits for the
next token and score the claim as the highest logit among the \textsc{true} variants
(\texttt{TRUE}, \texttt{True}, with and without a leading space) minus the highest among the
\textsc{false} variants (Eq.~\ref{eq:score}). The score is the model's relative preference
between the two answer words; AUROC ranks claims by it, and accuracy thresholds it at the
median score of the fit half.

\paragraph{Example claims.}
A true claim and its false counterpart differ only in the object: the false one takes a
distractor from BMLAMA's candidate list, at the same position in every language
(Table~\ref{tab:claimexamples}). The Polish renderings show the limitation noted in
Section~\ref{sec:limitations}: BMLAMA fills templates without declining nouns, so
the object keeps its dictionary form (\emph{z Bułgaria}) where Polish grammar requires the
instrumental case.

\begin{table}[ht]
\centering
\caption{Two claims of the dataset in the five Latin-script languages.}
\label{tab:claimexamples}
\small
\setlength{\tabcolsep}{4pt}
\begin{tabular}{l l l}
\toprule
Language & true & false \\
\midrule
English    & Colorado shares border with Wyoming.       & Croatia shares border with Bulgaria. \\
Portuguese & Colorado compartilha fronteira com Wyoming. & Croácia compartilha fronteira com Bulgária. \\
Polish     & Kolorado graniczy z Wyoming.               & Chorwacja graniczy z Bułgaria. \\
German     & Colorado teilt die Grenze mit Wyoming.     & Kroatien teilt die Grenze mit Bulgarien. \\
Turkish    & Colorado, Wyoming ile sınır paylaşıyor.    & Hırvatistan, Bulgaristan ile sınır paylaşıyor. \\
\bottomrule
\end{tabular}
\end{table}

\paragraph{Claim types.}
The 1,500 claims come from 18 BMLAMA relation templates; Table~\ref{tab:relations} gives each
with its number of true and false claims. Three relations, twin cities, shared borders and
official language, make up two thirds of the set.

\begin{table}[ht]
\centering
\caption{Relation templates in the claim set, with the number of claims ($S$ subject, $O$
object).}
\label{tab:relations}
\small
\begin{tabular}{l rrr}
\toprule
Template & true & false & total \\
\midrule
$S$ and $O$ are twin cities.                 & 290 & 283 & 573 \\
$S$ shares border with $O$.                  & 163 & 145 & 308 \\
The official language of $S$ is $O$.         & 100 & 117 & 217 \\
$S$ is the capital of $O$.                   &  55 &  62 & 117 \\
$S$ is located in $O$.                       &  51 &  50 & 101 \\
The capital of $S$ is $O$.                   &  43 &  33 &  76 \\
$S$ is part of $O$.                          &   6 &  11 &  17 \\
$S$ consists of $O$.                         &   8 &   8 &  16 \\
$S$ used to work in $O$.                     &   8 &   8 &  16 \\
$S$ works in the field of $O$.               &   5 &   8 &  13 \\
The native language of $S$ is $O$.          &   7 &   5 &  12 \\
$S$ is named after $O$.                      &   4 &   8 &  12 \\
$S$ is affiliated with the $O$ religion.     &   5 &   5 &  10 \\
$S$ is $O$ citizen.                          &   0 &   3 &   3 \\
$S$ used to communicate in $O$.              &   2 &   1 &   3 \\
$S$ is a $O$.                                &   2 &   1 &   3 \\
$S$ has the position of $O$.                 &   0 &   2 &   2 \\
$S$ died in $O$.                             &   1 &   0 &   1 \\
\midrule
Total                                        & 750 & 750 & 1{,}500 \\
\bottomrule
\end{tabular}
\end{table}
\FloatBarrier

\section{Non-Uniqueness of the Fitted Rotation}
\label{app:nonunique}

\paragraph{The issue.}
We fit on $n = 750$ parallel claims while the residual width $d$ ranges from
$3{,}072$ to $8{,}192$. The cross-covariance $\mathbf{M} = \mathbf{A}^\top\mathbf{B}$
therefore has rank $r \le n-1 < d$, so its singular value decomposition has at least
$d - r$ zero singular values, and the corresponding columns of $\mathbf{U}$ and
$\mathbf{V}$ are an arbitrary orthonormal basis of their null spaces. The minimiser of
the Procrustes objective is consequently not unique.

\begin{proposition}[The family of optimal rotations]
\label{prop:family}
Let $r = \operatorname{rank}(\mathbf{M})$ and partition the full SVD as
$\mathbf{U} = [\mathbf{U}_r\;\mathbf{U}_0]$, $\mathbf{V} = [\mathbf{V}_r\;\mathbf{V}_0]$
by the nonzero and zero singular values. For every $\mathbf{W}$ in the orthogonal group
$\mathcal{O}(d-r)$, the matrix
\begin{equation}
  \mathbf{R}_{\mathbf{W}}
  = \mathbf{U}_r\mathbf{V}_r^\top + \mathbf{U}_0\,\mathbf{W}\,\mathbf{V}_0^\top
  \label{eq:family}
\end{equation}
is orthogonal and attains the same minimum of
$\lVert\mathbf{A}\mathbf{R}-\mathbf{B}\rVert_F$. The set of minimisers is therefore a
manifold of dimension $(d-r)(d-r-1)/2$.
\end{proposition}

\begin{proof}
Since
$\lVert\mathbf{A}\mathbf{R}-\mathbf{B}\rVert_F^2
= \lVert\mathbf{A}\rVert_F^2+\lVert\mathbf{B}\rVert_F^2-2\operatorname{tr}(\mathbf{R}^\top\mathbf{M})$,
minimising the residual is maximising $\operatorname{tr}(\mathbf{R}^\top\mathbf{M})$.
For any orthogonal $\mathbf{R}$, the matrix $\mathbf{W}' = \mathbf{U}^\top\mathbf{R}\mathbf{V}$ is
orthogonal and $\operatorname{tr}(\mathbf{R}^\top\mathbf{M}) = \sum_j \sigma_j W'_{jj}
\le \sum_{j\le r}\sigma_j$, with equality exactly when $W'_{jj} = 1$ for every $j \le r$.
Because the rows and columns of $\mathbf{W}'$ are unit vectors, this forces
$\mathbf{W}' = \operatorname{diag}(\mathbf{I}_r, \mathbf{W})$ with $\mathbf{W}\in\mathcal{O}(d-r)$,
which is \eqref{eq:family}. Distinct $\mathbf{W}$ give distinct $\mathbf{R}_{\mathbf{W}}$, and
$\mathcal{O}(d-r)$ has dimension $(d-r)(d-r-1)/2$.
\end{proof}

At $d = 4096$ and $r = 749$ this family has dimension
$3347 \times 3346 / 2 \approx 5.6\times 10^{6}$. All of its members fit the 750 training
pairs equally well, and they differ only on directions the data never constrains.

\paragraph{The canonical choice.}
A well-defined choice among them is the optimal rotation closest to the identity,
\begin{equation}
  \mathbf{R}^\star
  = \arg\min_{\mathbf{W}\in\mathcal{SO}(d-r)}\;
    \lVert\mathbf{R}_{\mathbf{W}} - \mathbf{I}\rVert_F ,
  \label{eq:canonical}
\end{equation}
restricted to $\det\mathbf{R} = +1$ so that the transform is a proper rotation rather
than a rotoreflection. Because
$\lVert\mathbf{R}-\mathbf{I}\rVert_F^2 = 2d - 2\operatorname{tr}(\mathbf{R})$, the
minimiser of \eqref{eq:canonical} is the maximiser of
$\operatorname{tr}(\mathbf{W}\mathbf{V}_0^\top\mathbf{U}_0)$, which is a Procrustes
problem in $\mathcal{O}(d-r)$ and is unique whenever
$\mathbf{V}_0^\top\mathbf{U}_0$ is nonsingular.

\begin{remark}
This choice is well defined, reproducible across linear-algebra implementations, and
always a proper rotation. It changes nothing on the $r$ directions the data fixes, so
the training residual is identical to that of any other optimal solution.
\end{remark}

\paragraph{What we verified.}
A refit changes $\mathbf{R}$ only on the $d - r$ directions the data leaves free. Held-out
activations do carry components in those directions, so held-out behaviour is not
guaranteed to be unchanged. We therefore measured, for each held-out claim, the share of its
centred activation's squared norm that lies inside the span of the fitting activations, and
refitted Llama-3B and Gemma-9B under \eqref{eq:canonical}, keeping the data, split, selected
layers and sequential fit of the main runs fixed. All numbers in the main text use the
implementation-default solution; these two refits are a check. The in-span share is
$99.1\%$ on Llama-3B and $98.8\%$ on Gemma-9B, and at least $98.3\%$ at every fitted layer
(equivalently, the out-of-span share reported in Section~\ref{sec:procrustes} is $0.9\%$
and $1.2\%$). Held-out non-English AUROC moves from $0.8274$ to $0.8277$ on Llama-3B and
from $0.8657$ to $0.8649$ on Gemma-9B, a difference of $+0.0003$ and $-0.0008$, far smaller
than the method's own gain on these models ($+0.087$ and $+0.028$ over the untouched
baseline). The canonical solution removes every reflection and stays far closer to the
identity (Table~\ref{tab:diagnostics}); Table~\ref{tab:perlayer} gives every fitted layer.

An indirect check points the same way. A Haar-random rotation reaches $0.658$ against the
method's $0.853$ and the untouched baseline's $0.820$. If the arbitrary component of
$\mathbf{R}$ were driving the result, the fitted rotation would sit near the random one.

\begin{table}[ht]
\centering
\caption{Diagnostics for the fitted transform over the 21 selected layers (7 languages
$\times$ 3 layers). $r$ is the mean numerical rank of $\mathbf{M}$ (singular values above
$10^{-6}\sigma_{\max}$); $\det\mathbf{R}=-1$ counts reflections; where two values are given
they are for the implementation-default / canonical solution. $\sqrt{2d}$ is the value
$\lVert\mathbf{R}-\mathbf{I}\rVert_F$ takes for a Haar-random rotation, shown for reference.}
\label{tab:diagnostics}
\small
\begin{tabular}{l cccccc}
\toprule
Model & $r$ & $\det\mathbf{R}=-1$ & $\lVert\mathbf{R}-\mathbf{I}\rVert_F$ & $\sqrt{2d}$ & in-span share & non-EN AUROC \\
\midrule
Llama-3B   & 737 & 12/21 / 0/21 & 76.5 / 38.7 & 78.4 & 99.1\% & 0.8274 / 0.8277 \\
Gemma-9B   & 745 & 5/21 / 0/21  & 82.1 / 36.6 & 84.7 & 98.8\% & 0.8657 / 0.8649 \\
\bottomrule
\end{tabular}
\end{table}

\begin{table}[ht]
\centering
\caption{Per-layer diagnostics for the two refitted models. $\det\mathbf{R}$ and
$\lVert\mathbf{R}-\mathbf{I}\rVert_F$ are given for the implementation-default / canonical
solution; in-span is the share of a centred held-out activation's squared norm inside the span
of the fitting activations, averaged over held-out claims.}
\label{tab:perlayer}
\scriptsize
\setlength{\tabcolsep}{4pt}
\begin{tabular}{l c c c c c c c}
\toprule
Language & Layer & rank$(\mathbf{M})$ & $\sigma_{\max}$ & $\det\mathbf{R}$ & $\lVert\mathbf{R}-\mathbf{I}\rVert_F$ & $\lVert\bm{\mu}^{\mathrm{en}}-\bm{\mu}^{\ell}\rVert$ & in-span \\
\midrule
\multicolumn{8}{l}{\textit{Llama-3B} ($d = 3{,}072$, $\sqrt{2d} = 78.4$)} \\
PT & 14 & 738 & $1.4{\times}10^{3}$ & $-1$ / $+1$ & 77.2 / 41.1 & 6.02 & 99.3\% \\
PT & 20 & 717 & $3.6{\times}10^{3}$ & $-1$ / $+1$ & 75.3 / 34.0 & 4.15 & 99.4\% \\
PT & 24 & 721 & $1.4{\times}10^{4}$ & $-1$ / $+1$ & 73.9 / 30.4 & 6.14 & 98.9\% \\
PL & 14 & 741 & $1.7{\times}10^{3}$ & $-1$ / $+1$ & 77.5 / 41.7 & 5.99 & 99.4\% \\
PL & 22 & 736 & $5.5{\times}10^{3}$ & $+1$ / $+1$ & 75.2 / 35.8 & 7.12 & 99.1\% \\
PL & 26 & 733 & $2.2{\times}10^{4}$ & $-1$ / $+1$ & 74.2 / 31.2 & 7.76 & 98.9\% \\
DE & 20 & 710 & $3.9{\times}10^{3}$ & $-1$ / $+1$ & 76.4 / 38.9 & 7.27 & 99.6\% \\
DE & 6 & 746 & $3.3{\times}10^{1}$ & $+1$ / $+1$ & 76.6 / 39.1 & 2.67 & 99.0\% \\
DE & 22 & 728 & $8.1{\times}10^{3}$ & $+1$ / $+1$ & 73.7 / 28.9 & 4.42 & 99.2\% \\
AR & 16 & 734 & $1.5{\times}10^{3}$ & $-1$ / $+1$ & 78.1 / 44.3 & 8.82 & 99.5\% \\
AR & 20 & 731 & $2.5{\times}10^{3}$ & $+1$ / $+1$ & 74.6 / 33.6 & 4.20 & 99.4\% \\
AR & 14 & 743 & $1.1{\times}10^{3}$ & $+1$ / $+1$ & 78.1 / 43.9 & 7.24 & 99.3\% \\
ZH & 20 & 743 & $2.9{\times}10^{3}$ & $-1$ / $+1$ & 78.0 / 44.4 & 10.03 & 99.3\% \\
ZH & 22 & 746 & $4.3{\times}10^{3}$ & $+1$ / $+1$ & 74.9 / 32.6 & 3.79 & 98.6\% \\
ZH & 12 & 746 & $3.6{\times}10^{2}$ & $-1$ / $+1$ & 78.1 / 43.5 & 5.54 & 98.7\% \\
RU & 14 & 744 & $1.8{\times}10^{3}$ & $+1$ / $+1$ & 77.7 / 41.7 & 4.85 & 99.3\% \\
RU & 20 & 730 & $3.9{\times}10^{3}$ & $-1$ / $+1$ & 75.2 / 35.2 & 4.06 & 99.5\% \\
RU & 12 & 746 & $3.4{\times}10^{2}$ & $-1$ / $+1$ & 77.9 / 42.4 & 4.36 & 98.7\% \\
TR & 14 & 746 & $1.2{\times}10^{3}$ & $+1$ / $+1$ & 77.8 / 43.1 & 5.83 & 99.3\% \\
TR & 12 & 746 & $2.9{\times}10^{2}$ & $-1$ / $+1$ & 77.9 / 43.2 & 5.05 & 98.9\% \\
TR & 6 & 746 & $4.4{\times}10^{1}$ & $+1$ / $+1$ & 77.5 / 42.9 & 3.17 & 98.9\% \\
\midrule
\multicolumn{8}{l}{\textit{Gemma-9B} ($d = 3{,}584$, $\sqrt{2d} = 84.7$)} \\
PT & 28 & 742 & $5.8{\times}10^{6}$ & $+1$ / $+1$ & 81.5 / 34.5 & 63.60 & 99.2\% \\
PT & 22 & 747 & $9.2{\times}10^{5}$ & $+1$ / $+1$ & 81.6 / 34.1 & 36.28 & 98.7\% \\
PT & 26 & 745 & $4.0{\times}10^{6}$ & $-1$ / $+1$ & 80.7 / 31.4 & 37.74 & 98.9\% \\
PL & 26 & 745 & $3.3{\times}10^{6}$ & $+1$ / $+1$ & 82.3 / 37.2 & 87.43 & 99.0\% \\
PL & 24 & 747 & $1.5{\times}10^{6}$ & $+1$ / $+1$ & 82.4 / 37.1 & 63.84 & 98.8\% \\
PL & 22 & 747 & $9.6{\times}10^{5}$ & $+1$ / $+1$ & 82.2 / 36.9 & 53.62 & 98.7\% \\
DE & 24 & 746 & $1.6{\times}10^{6}$ & $+1$ / $+1$ & 81.3 / 33.7 & 56.92 & 99.0\% \\
DE & 22 & 748 & $9.7{\times}10^{5}$ & $-1$ / $+1$ & 81.4 / 33.6 & 46.75 & 98.9\% \\
DE & 28 & 732 & $6.2{\times}10^{6}$ & $-1$ / $+1$ & 80.9 / 31.9 & 46.13 & 99.4\% \\
AR & 22 & 746 & $9.2{\times}10^{5}$ & $+1$ / $+1$ & 83.0 / 38.7 & 65.23 & 98.6\% \\
AR & 26 & 745 & $3.4{\times}10^{6}$ & $+1$ / $+1$ & 81.9 / 35.5 & 63.67 & 98.9\% \\
AR & 24 & 746 & $1.5{\times}10^{6}$ & $+1$ / $+1$ & 80.9 / 31.8 & 36.51 & 98.7\% \\
ZH & 38 & 746 & $1.4{\times}10^{7}$ & $+1$ / $+1$ & 83.0 / 41.3 & 285.07 & 98.6\% \\
ZH & 22 & 748 & $7.5{\times}10^{5}$ & $+1$ / $+1$ & 83.1 / 40.6 & 76.16 & 98.3\% \\
ZH & 36 & 745 & $1.0{\times}10^{7}$ & $+1$ / $+1$ & 82.8 / 40.6 & 141.17 & 98.5\% \\
RU & 24 & 747 & $1.6{\times}10^{6}$ & $+1$ / $+1$ & 82.7 / 37.8 & 60.57 & 98.8\% \\
RU & 22 & 748 & $9.5{\times}10^{5}$ & $-1$ / $+1$ & 82.7 / 37.8 & 51.47 & 98.7\% \\
RU & 26 & 745 & $3.6{\times}10^{6}$ & $+1$ / $+1$ & 81.5 / 35.0 & 53.59 & 99.0\% \\
TR & 28 & 738 & $5.2{\times}10^{6}$ & $-1$ / $+1$ & 83.0 / 39.6 & 104.04 & 99.2\% \\
TR & 26 & 745 & $3.6{\times}10^{6}$ & $+1$ / $+1$ & 83.1 / 39.6 & 101.91 & 99.0\% \\
TR & 24 & 747 & $1.5{\times}10^{6}$ & $+1$ / $+1$ & 83.1 / 39.2 & 77.99 & 98.8\% \\
\bottomrule
\end{tabular}
\end{table}

\section{Why an Unconstrained Map Discards Information}
\label{app:ridge}

\begin{proposition}[Out-of-span annihilation]
\label{prop:annihilate}
Let $\mathbf{W} = (\mathbf{A}^\top\mathbf{A}+\lambda\mathbf{I})^{-1}\mathbf{A}^\top\mathbf{B}$
be the ridge map used in our unconstrained-map control, with $\lambda \ge 0$. For any
activation $\mathbf{h}$ orthogonal to every row of $\mathbf{A}$ we have
$\mathbf{h}\mathbf{W} = \mathbf{0}$. Since $n < d$, the space of such $\mathbf{h}$ has
dimension at least $d - n$, so $\mathbf{W}$ annihilates a subspace of that dimension. The
orthogonal $\mathbf{R}$ is invertible and annihilates nothing.
\end{proposition}

\begin{proof}
Write $\mathbf{W} = \mathbf{A}^\top\mathbf{C}$ for
$\mathbf{C} = (\mathbf{A}\mathbf{A}^\top+\lambda\mathbf{I})^{-1}\mathbf{B}$. If
$\mathbf{h}$ is orthogonal to the rows of $\mathbf{A}$ then $\mathbf{h}\mathbf{A}^\top=\mathbf{0}$,
hence $\mathbf{h}\mathbf{W}=\mathbf{0}$. The set of such $\mathbf{h}$ is the orthogonal
complement of a space of dimension at most $n$. Orthogonality of $\mathbf{R}$ gives
$\mathbf{R}^{-1}=\mathbf{R}^\top$.
\end{proof}

A held-out activation therefore loses whatever fraction of itself lies outside the span of
the 750 fitting activations before the read-out head ever sees it. In our data that fraction
is small (at most $1.7\%$, Appendix~\ref{app:nonunique}), so the deletion alone does not
explain the unconstrained-map arm scoring $0.786$, below the untouched baseline of $0.820$;
the ridge map also rescales the directions inside the span, which a rotation never does. The
point of the proposition is structural: an unconstrained map can destroy information, and a
rotation cannot.

\paragraph{A two-claim illustration.}
Take $n = 2$ and $d = 3$, with the read-out reading the first coordinate. Let the centred
non-English and English activations be
\begin{equation}
  \mathbf{A} = \begin{pmatrix} -1 & 2 & 0 \\ -1 & -2 & 0 \end{pmatrix},
  \qquad
  \mathbf{B} = \begin{pmatrix} 2 & 1 & 0 \\ -2 & 1 & 0 \end{pmatrix}.
\end{equation}
The least-squares map is
\begin{equation}
  \mathbf{W}^\star = \mathbf{A}^{+}\mathbf{B}
  = \begin{pmatrix} 0 & -1 & 0 \\ 1 & 0 & 0 \\ 0 & 0 & 0 \end{pmatrix},
\end{equation}
which maps both fitting claims onto their partners exactly. Its third row and third column
are zero: every component a held-out claim carries along the third axis is deleted. Its
singular values are $(1,1,0)$. The orthogonal solution for the same data is
\begin{equation}
  \mathbf{R} = \begin{pmatrix} 0 & -1 & 0 \\ 1 & 0 & 0 \\ 0 & 0 & 1 \end{pmatrix},
\end{equation}
with singular values $(1,1,1)$: the same action on the two fitted directions, and the
identity on the direction the data says nothing about.

\section{What the Transform Can and Cannot Do}
\label{app:information}

\begin{proposition}[Information is preserved exactly]
\label{prop:info}
Let $Y$ be the truth value of a claim and $H$ the layer-$l$ activation. The \textup{\textsc{RoSh}}
map $T(\mathbf{h}) = (\mathbf{h}-\bm{\mu}^{\ell})\mathbf{R}+\bm{\mu}^{\mathrm{en}}$ is a
bijection with inverse
$T^{-1}(\mathbf{z}) = (\mathbf{z}-\bm{\mu}^{\mathrm{en}})\mathbf{R}^\top+\bm{\mu}^{\ell}$,
so $I(Y;T(H)) = I(Y;H)$ and the Bayes error of predicting $Y$ from $T(H)$ equals that of
predicting $Y$ from $H$.
\end{proposition}

\begin{proof}
A bijective bi-measurable map induces $\sigma(T(H))=\sigma(H)$, and mutual information
depends on its arguments only through the generated $\sigma$-algebras.
\end{proof}

\begin{remark}[No map adds information]
\label{rem:dpi}
By the data-processing inequality, $I(Y;f(H)) \le I(Y;H)$ for any fixed $f$, so no
intervention of this kind, ours or any baseline's, can create information about a claim's
truth. The distinction is therefore not that others could inject knowledge. It is that a
non-invertible map can \emph{remove} it, while ours provably cannot.
\end{remark}

\begin{remark}[The gain cannot be denoising]
\label{rem:denoise}
This distinction has a consequence that is easy to miss. A map that deletes directions can
improve a fixed downstream reader simply by removing clutter, which is denoising rather
than re-orientation, and a score improvement from such a map mixes the two effects. Because
\textup{\textsc{RoSh}} deletes nothing, any improvement it produces cannot be denoising. It can only
be a change of orientation, which is why we read the gain as evidence about where the
model's knowledge was pointing rather than about what was removed.
\end{remark}

\begin{proposition}[English is preserved exactly]
\label{prop:english}
Fitting the transform for English against itself gives $\mathbf{A}=\mathbf{B}$, so
$\mathbf{M}=\mathbf{A}^\top\mathbf{A}$ is symmetric positive semidefinite and its left and
right singular vectors coincide on the $r$ directions the data determine. Every optimal
$\mathbf{R}$ therefore acts as the identity on those directions, and the canonical solution
\eqref{eq:canonical} is exactly $\mathbf{R}=\mathbf{I}$. The shift is
$\bm{\mu}^{\mathrm{en}}-\bm{\mu}^{\mathrm{en}}=\mathbf{0}$, so $T$ is the identity and the
English output is unchanged.
\end{proposition}

This is the exact null reported in Table~\ref{tab:arms}, where fitting English to itself
moves the score by $0.000$. A weight update shared across languages offers
no such guarantee, since it generally changes the function on English inputs as well.

\section{On Measuring the Angle Between the Truth Direction and the Read-Out}
\label{app:angle}

An earlier version of this work described the non-English failure as the truth direction
sitting close to orthogonal to the direction the output head reads. Direct measurement does
not support that description, and we report the reason here.

On a separate 570-claim diagnostic set, the angle between the final-layer truth direction
and the read-out direction is close to $90^\circ$ in every language, English included. In
$d$ dimensions the cosine between two arbitrary directions concentrates near zero with
standard deviation approximately $1/\sqrt{d}$, which is $0.016$ at $d = 4096$, or
$89.1^\circ$. An angle of about $90^\circ$ is therefore the value this measurement takes
when it carries no information, and it cannot distinguish a language where the read-out
succeeds from one where it fails.

The measurement is also not the quantity that governs separation. For a read-out direction
$\mathbf{u}$ and a class-mean difference $\mathbf{t}$, the separation of the score is
\begin{equation}
  \frac{\mathbf{u}^\top\mathbf{t}}{\sigma_{\mathbf{u}}},
  \label{eq:effectsize}
\end{equation}
where $\sigma_{\mathbf{u}}$ is the within-class spread of the score. A cosine of $0.05$,
which is $87^\circ$, yields excellent separation when $\lVert\mathbf{t}\rVert$ is large
relative to $\sigma_{\mathbf{u}}$, and none when it is not. Two languages can share an
angle and differ entirely in \eqref{eq:effectsize}.

We therefore report the comparison that does govern the effect, per language: the AUROC a
linear probe attains on the layer-$l$ activation, against the AUROC the model's own
read-out attains on the same claims. Table~\ref{tab:probe} gives both. The probe stays high
in every language shown while the read-out falls further short of it outside English; English
itself leaves a gap, which Section~\ref{sec:decomposition} separates from the
language-specific part.

Table~\ref{tab:probeall} gives the same comparison for all eight models, with English
alongside, and is the basis of Section~\ref{sec:decomposition}.

\begin{table}[ht]
\centering
\caption{Read-out against a linear probe, held out, for all eight models. Read-out and
\textsc{RoSh} are from Table~\ref{tab:main}. The probe is an L2-regularised logistic
regression on the fit half, scored on the held-out half, at the best candidate layer; the
non-English probe averages the seven languages. Gap is probe minus read-out. Closed is
$(\textsc{RoSh} - \text{baseline}) / \text{gap before}$; the Mean row's share is computed
from the mean gap.}
\label{tab:probeall}
\small
\setlength{\tabcolsep}{4pt}
\resizebox{\textwidth}{!}{%
\begin{tabular}{l ccc ccc cc c}
\toprule
& \multicolumn{3}{c}{English} & \multicolumn{3}{c}{non-English} & \multicolumn{2}{c}{non-EN gap} & \\
\cmidrule(lr){2-4}\cmidrule(lr){5-7}\cmidrule(lr){8-9}
Model & read-out & probe & gap & baseline & \textsc{RoSh} & probe & before & after & closed \\
\midrule
Llama-3B         & 0.857 & 0.923 & $+0.066$ & 0.740 & 0.829 & 0.893 & $+0.153$ & $+0.064$ & 58\% \\
Mistral-7B       & 0.834 & 0.927 & $+0.093$ & 0.704 & 0.809 & 0.857 & $+0.153$ & $+0.048$ & 69\% \\
Gemma-9B         & 0.878 & 0.941 & $+0.063$ & 0.838 & 0.865 & 0.928 & $+0.090$ & $+0.063$ & 30\% \\
Mistral-24B      & 0.888 & 0.939 & $+0.051$ & 0.866 & 0.878 & 0.931 & $+0.065$ & $+0.053$ & 18\% \\
Gemma-27B        & 0.860 & 0.928 & $+0.068$ & 0.833 & 0.844 & 0.911 & $+0.078$ & $+0.067$ & 14\% \\
Qwen-32B         & 0.885 & 0.935 & $+0.050$ & 0.879 & 0.889 & 0.931 & $+0.052$ & $+0.042$ & 19\% \\
Llama-70B        & 0.879 & 0.948 & $+0.069$ & 0.865 & 0.883 & 0.935 & $+0.070$ & $+0.052$ & 26\% \\
DeepSeek-67B     & 0.875 & 0.942 & $+0.067$ & 0.861 & 0.872 & 0.891 & $+0.030$ & $+0.019$ & 37\% \\
\midrule
Mean             & 0.870 & 0.935 & $+0.066$ & 0.823 & 0.859 & 0.910 & $+0.086$ & $+0.051$ & 41\% \\
\bottomrule
\end{tabular}}
\end{table}

\section{A Worked Example of \textsc{RoSh} in Three Dimensions}
\label{app:toy}

The mechanism is easiest to see at $d = 3$, where the read-out reads the first coordinate,
so the score of an activation $\mathbf{h}$ is $s(\mathbf{h}) = h_1$.

\paragraph{Setup.}
Four claims, two true and two false. In English,
\begin{equation}
  \mathbf{h}^{\mathrm{en}}:\quad
  (2,1,0),\;(2,-1,0)\;\text{[true]},\qquad
  (-2,1,0),\;(-2,-1,0)\;\text{[false]},
\end{equation}
so the score is $+2$ for both true claims and $-2$ for both false claims, and the read-out
separates them perfectly. In Arabic the same four claims give
\begin{equation}
  \mathbf{h}^{\mathrm{ar}}:\quad
  (4,7,0),\;(6,7,0)\;\text{[true]},\qquad
  (4,3,0),\;(6,3,0)\;\text{[false]}.
\end{equation}
The scores are now $\{4,6\}$ for the true claims and $\{4,6\}$ for the false ones, so the
read-out is at chance. The separation has not disappeared: it has moved to the second
coordinate, where true claims read $7$ and false claims read $3$. A probe along the second
axis recovers the truth perfectly while the read-out recovers nothing, which is the
situation of Appendix~\ref{app:angle} in miniature.

\paragraph{Fitting.}
The centroids are $\bm{\mu}^{\mathrm{en}} = (0,0,0)$ and $\bm{\mu}^{\mathrm{ar}} = (5,5,0)$.
Centring gives
\begin{equation}
  \mathbf{A} = \begin{pmatrix} -1 & 2 & 0\\ 1 & 2 & 0\\ -1 & -2 & 0\\ 1 & -2 & 0\end{pmatrix},
  \qquad
  \mathbf{B} = \begin{pmatrix} 2 & 1 & 0\\ 2 & -1 & 0\\ -2 & 1 & 0\\ -2 & -1 & 0\end{pmatrix},
  \qquad
  \mathbf{M} = \mathbf{A}^\top\mathbf{B}
  = \begin{pmatrix} 0 & -4 & 0\\ 16 & 0 & 0\\ 0 & 0 & 0\end{pmatrix}.
\end{equation}
The singular values of $\mathbf{M}$ are $16$, $4$ and $0$; the zero reflects that all four
claims lie in the plane $h_3 = 0$, so the data constrains only two of the three directions.
The solution is
\begin{equation}
  \mathbf{R} = \mathbf{U}\mathbf{V}^\top
  = \begin{pmatrix} 0 & -1 & 0\\ 1 & 0 & 0\\ 0 & 0 & 1\end{pmatrix},
  \qquad \det\mathbf{R} = +1,
\end{equation}
a rotation of $-90^\circ$ in the $(1,2)$ plane.

\paragraph{Applying it.}
For the first Arabic claim,
\begin{equation}
  (4,7,0) - \bm{\mu}^{\mathrm{ar}} = (-1,2,0),
  \qquad
  (-1,2,0)\,\mathbf{R} = (2,1,0),
  \qquad
  (2,1,0) + \bm{\mu}^{\mathrm{en}} = (2,1,0),
\end{equation}
which is exactly the English activation for that claim. The same holds for the other three,
so the score returns to $\pm 2$ and the read-out separates the classes again.

\paragraph{Why a shift alone cannot do this.}
The shift adds the single vector $\bm{\mu}^{\mathrm{en}}-\bm{\mu}^{\mathrm{ar}} = (-5,-5,0)$
to every activation. Every point moves by the same amount in the same direction, so the line
joining the true and false groups keeps its direction: a translation cannot change an
orientation. In the real models the shift nonetheless contributes, because the corrected
activation passes through further non-linear layers which do not commute with translation,
and Table~\ref{tab:arms} attributes $43\%$ of the gain to the shift and $74\%$ to the
rotation, the two not being additive.

\paragraph{Why the pairing matters.}
$\mathbf{M} = \mathbf{A}^\top\mathbf{B}$ is a sum of outer products of each Arabic claim
with \emph{its own} English counterpart. Permuting the rows of $\mathbf{B}$ leaves both
clouds and both centroids unchanged but destroys that correspondence, and with it the only
information the rotation uses. This is the scrambled-pairing control, which falls to
$0.621$ in Table~\ref{tab:arms}. A shift computed from the means would be entirely
unaffected by the same permutation.

\section{External Benchmarks}
\label{app:external}

We also run \textsc{RoSh} on the two benchmarks of the closest inference-time method,
latent-space intervention (LSI) \citep{ghorbanpour2026lsi}, using their data and their metric
code, and compare with the numbers they report. On KLAR, averaged over three models, it adds
$+5.63$ accuracy and $+5.48$ agreement with English, against $+0.44$ and $+1.02$ for LSI. On
mParaRel, where LSI reports a single model, Aya, \textsc{RoSh} adds $+2.24$ agreement with
English against $+0.27$ (Table~\ref{tab:external}).

\begin{table}[ht]
\centering
\caption{External benchmarks, in \%. The LSI columns are exactly as reported by
\citet{ghorbanpour2026lsi} (their Tables 1, 4 and 5): their baseline and their variant with the
highest agreement with English, ties broken by accuracy (Qwen3-8B mean-shift, Llama-3.1-8B
AE + mean-shift, Aya-8B AE on KLAR and AE + PCA on mParaRel); on mParaRel they report Aya only,
so we do the same. The \textsc{RoSh} columns are our runs, with accuracy the non-English mean
over Arabic, Dutch, Russian and Chinese. The two methods are evaluated on different splits, so
the gains $\Delta$ are the comparable quantity.}
\label{tab:external}
\small
\setlength{\tabcolsep}{3pt}
\resizebox{\textwidth}{!}{%
\begin{tabular}{l l ccc ccc ccc ccc}
\toprule
& & \multicolumn{6}{c}{accuracy} & \multicolumn{6}{c}{agreement with English} \\
\cmidrule(lr){3-8}\cmidrule(lr){9-14}
& & \multicolumn{3}{c}{LSI (reported)} & \multicolumn{3}{c}{\textsc{RoSh} (ours)} & \multicolumn{3}{c}{LSI (reported)} & \multicolumn{3}{c}{\textsc{RoSh} (ours)} \\
\cmidrule(lr){3-5}\cmidrule(lr){6-8}\cmidrule(lr){9-11}\cmidrule(lr){12-14}
Benchmark & Model & base & LSI & $\Delta$ & base & \textsc{RoSh} & $\Delta$ & base & LSI & $\Delta$ & base & \textsc{RoSh} & $\Delta$ \\
\midrule
KLAR     & Qwen3-8B     & 87.51 & 88.25 & $+0.74$ & 85.53 & 90.17 & $+4.64$ & 86.64 & 87.38 & $+0.74$ & 86.78 & 91.57 & $+4.79$ \\
KLAR     & Llama-3.1-8B & 78.21 & 79.05 & $+0.84$ & 74.06 & 80.21 & $+6.15$ & 76.31 & 78.29 & $+1.98$ & 75.40 & 81.47 & $+6.07$ \\
KLAR     & Aya-8B       & 85.59 & 85.32 & $-0.27$ & 83.55 & 89.64 & $+6.09$ & 85.09 & 85.43 & $+0.34$ & 84.79 & 90.36 & $+5.57$ \\
\cmidrule(lr){2-14}
KLAR     & mean         &       &       & $+0.44$ &       &       & $+5.63$ &       &       & $+1.02$ &       &       & $+5.48$ \\
\midrule
mParaRel & Aya-8B       & 85.91 & 86.37 & $+0.46$ & -- & -- & -- & 86.70 & 86.97 & $+0.27$ & 86.70 & 88.94 & $+2.24$ \\
\bottomrule
\end{tabular}}
\end{table}

\section{Robustness to the Translation System}
\label{app:mt}

The mParaRel templates were produced by five machine-translation systems (Bing, Google,
M2M100, mBART-50 and OPUS-MT) and then reviewed by humans, so on that benchmark the gain could
in principle reflect translation quality rather than model behaviour. We split its Qwen3-8B
test claims by the system that produced each template and report the correction per
translation system, restricted to claims the model answers correctly in English, which fixes English accuracy at $100\%$ in every bucket and removes
the difference in item difficulty that otherwise confounds the comparison. Intervals are
$95\%$ bootstrap over claims. The two smallest buckets are merged.

The table is presented as a check that the correction is not harmful under any translation
source rather than as a measurement of how translation quality drives the effect, since
larger gains under weaker translation are also what greater headroom alone would produce.
Over all systems the correction adds $+1.36$ points ($95\%$ CI $+0.66$ to $+2.10$), while the
scrambled-correspondence control lowers agreement with English by $1.60$ points, and no system
shows a significant loss. Without the English-correct restriction the overall change is
$+1.53$ points ($95\%$ CI $+0.80$ to $+2.30$).

\begin{table}[ht]
\centering
\caption{Correction by translation system on mParaRel (Qwen3-8B, held out), restricted to
claims answered correctly in English. Non-English accuracy in \%; $\Delta$ in points with a
$95\%$ bootstrap interval.}
\label{tab:mt}
\small
\begin{tabular}{l c cc c}
\toprule
MT system & $N$ & non-EN baseline & non-EN \textsc{RoSh} & $\Delta$ [95\% CI] \\
\midrule
Google            & 90  & 94.7 & 95.0 & $+0.28$ [$-0.56$, $+1.11$] \\
Bing              & 279 & 86.0 & 88.5 & $+2.51$ [$+1.16$, $+3.85$] \\
OPUS-MT           & 180 & 92.1 & 91.9 & $-0.14$ [$-1.11$, $+0.83$] \\
mBART50 + m2m100  & 59  & 72.5 & 74.6 & $+2.12$ [$+0.00$, $+4.24$] \\
\midrule
All               & 608 & 87.8 & 89.1 & $+1.36$ [$+0.66$, $+2.10$] \\
\bottomrule
\end{tabular}
\end{table}

\section{Fitting Protocol}
\label{app:protocol}

Figure~\ref{fig:lookup_build} sets out the procedure of Sections~\ref{sec:procrustes} and~\ref{sec:layer_selection} end to end. Everything to the left of the held-out block happens on the 750 fitting claims: activations are captured, candidate layers are swept, each candidate is ranked on an inner fold, the rotation is fitted in closed form, and the sweep repeats with the correction in place. The output is one recipe per language, at most three layers, each with its rotation and its centroid pair. The held-out half enters once, at the end, to produce every number in Section~\ref{sec:results}.

\begin{figure}[!tb]
\centering
\includegraphics[width=\textwidth]{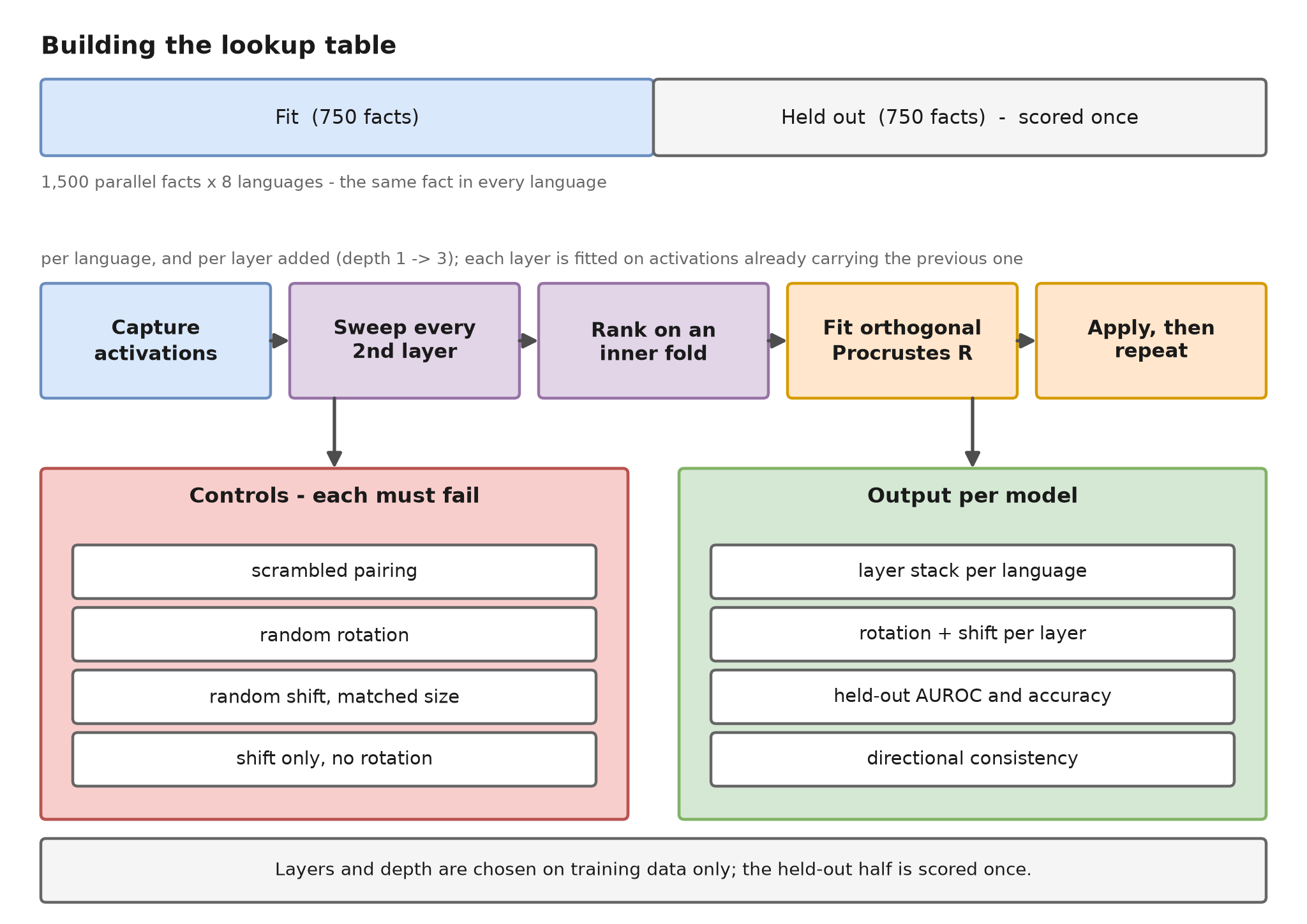}
\caption{\textbf{Building the lookup table.}
From 1{,}500 parallel facts ($\times\,8$ languages), we reserve half for fitting and half for held-out evaluation (scored once).
On the fit split, for each language we: (1) capture residual-stream activations at every second or third layer between $15\%$ and $97\%$ of depth, reducing the candidate space while retaining coverage across network depth, (2) rank candidate layers by AUROC on an inner fold of the fit split, and (3) greedily select up to $D{=}3$ layers, fitting the orthogonal Procrustes rotation $\R$ at each step.
This shallow forward search avoids the combinatorial cost of exhaustively evaluating layer combinations while allowing multiple interventions when they provide complementary gains.
Eleven controls (Section~\ref{sec:controls}), among them scrambled pairing, random rotation, random shift and random shift + rotation, must each score below the full method, confirming the result is not an artefact.
The held-out half is scored exactly once to report final accuracy and agreement.}
\label{fig:lookup_build}
\end{figure}

\section{Inference Algorithm}
\label{app:algorithm}

Algorithm~\ref{alg:rosh} gives the inference procedure of Section~\ref{sec:layer_selection}
in full. Only the answer position is transformed, and only at the layers listed in the
recipe for the claim's language.

\begin{algorithm}[!htbp]
\caption{RoSh inference}
\label{alg:rosh}
\begin{algorithmic}[1]
  \Require Claim $x$ in language $\ell$; recipe $\{(l_k, \R_k, \bmu_k^\ell, \bmu_k^{\mathrm{en}})\}_{k=1}^{K}$
  \Ensure  Corrected verdict (TRUE / FALSE)
  \For{layer $l = 0$ to $L-1$}
    \State $\h_l \gets \textsc{TransformerBlock}_l(\h_{l-1})$
    \If{$l \in \{l_1, \ldots, l_K\}$}
      \State $\h_l[\text{ans}] \gets
             (\h_l[\text{ans}] - \bmu_k^\ell)\,\R_k + \bmu_k^{\mathrm{en}}$
             \Comment{Eq.~\ref{eq:rosh}}
    \EndIf
  \EndFor
  \State \Return $\argmax_{v \in \{\text{TRUE},\text{FALSE}\}} p(v \mid \h_L)$
\end{algorithmic}
\end{algorithm}

\section{The Control Arms}
\label{app:controls}

Each control uses the same data and the same fitting procedure as the method but breaks
exactly one assumption. Table~\ref{tab:arms} gives the results.

\paragraph{Component ablations.}
Three arms isolate the two components of the transform:
\begin{itemize}
  \item \textbf{Shift only} (Eq.~\ref{eq:components}): translates the cloud without rotating it. Tests whether mean re-centring alone is sufficient.
  \item \textbf{Rotation only} (Eq.~\ref{eq:components}): rotates the cloud about its own centroid without shifting. Tests whether the rotation alone recovers the signal.
  \item \textbf{Unconstrained linear map}: replaces the orthogonal $\mathbf{R}$ with the ridge-regularised least-squares map $\mathbf{W} = (\mathbf{A}^\top\mathbf{A} + \lambda\mathbf{I})^{-1}\mathbf{A}^\top\mathbf{B}$, with $\lambda$ set to $10^{-2}$ times the mean eigenvalue of $\mathbf{A}^\top\mathbf{A}$. Because $n < d$, $\mathbf{W}$ sends every direction outside the span of the $n$ fitting activations to zero (Appendix~\ref{app:ridge}), so it discards whatever part of a held-out activation lies there, while the orthogonal $\mathbf{R}$ is invertible and discards nothing. Tests whether the orthogonality constraint is what lets the map generalise.
\end{itemize}

\paragraph{Correspondence controls.}
Two arms test whether the gain comes from the claim-level pairing:
\begin{itemize}
  \item \textbf{Scrambled pairing}: each non-English claim is paired with a \emph{random} English claim before fitting $\mathbf{R}$. The two clouds are the same; only the correspondence is destroyed.
  \item \textbf{English, rows unpaired}: scrambles pairings \emph{within English itself}, with no translation involved. This rules out any explanation that does not require a cross-lingual signal.
\end{itemize}

\paragraph{Randomisation controls.}
Three arms replace the fitted quantities with matched-size random ones:
\begin{itemize}
  \item \textbf{Random rotation}: a Haar-random orthogonal matrix of the same dimensionality.
  \item \textbf{Random shift}: a displacement of the same Euclidean norm as the real shift, pointing in a uniformly random direction.
  \item \textbf{Random shift + rotation}: both randomised simultaneously, the strongest ``any perturbation helps'' null.
\end{itemize}

\paragraph{Structural controls.}
Three arms verify that the transform is not exploiting a degree of freedom unrelated to language:
\begin{itemize}
  \item \textbf{Right map, wrong layer}: applies the correctly fitted $\mathbf{R}$ at an arbitrary layer instead of the selected one. Tests whether the layer choice matters.
  \item \textbf{Inverse map, on English}: applies the \emph{inverse} transform $\mathbf{R}^\top$ to English activations, mapping them toward the non-English frame. English AUROC should \emph{degrade} if the map is directional.
  \item \textbf{English to itself}: pairs each English claim with itself and fits $\mathbf{R}$. Every optimal solution is then the identity on the directions the data determine (Proposition~\ref{prop:english}), so the score must not change, a sanity null.
\end{itemize}

\noindent
A control passes when it scores below the full method; the three arms that act on English are
compared with the English baseline instead. All eleven pass when averaged over the eight models
(Table~\ref{tab:arms}).

\section{Tokeniser Verification}
\label{app:tokenizer}

DeepSeek-67B initially resolved to a broken slow tokeniser that discarded non-Latin script
entirely: Chinese claims tokenised to zero tokens, so every prompt in Arabic, Chinese and
Russian was identical and empty. Loading the fast tokeniser from \texttt{tokenizer.json}
fixes it, and the Chinese baseline rises from $0.501$ to $0.855$, Russian from $0.517$ to
$0.876$. All numbers in the paper use the corrected tokeniser, and every model was checked
for the same fault by round-tripping each prompt through decode. Its baseline accuracy and
directional consistency are recomputed from the arm run rather than the earlier threshold
file, which was written before the fix and understated both ($0.616$ and $0.657$ against the
corrected $0.748$ and $0.882$). All eight models are reported at depth 3.

\FloatBarrier
\section{Selected Layers}
\label{app:layers}

Table~\ref{tab:layers} lists the layer stack chosen for each language, in selection order.

\begin{table}[ht]
\centering
\caption{The layer stack chosen for each language, and the model's layer count.}
\label{tab:layers}
\small
\setlength{\tabcolsep}{5pt}
\resizebox{\textwidth}{!}{%
\begin{tabular}{l c ccccccc}
\toprule
Model & depth & PT & PL & DE & AR & ZH & RU & TR \\
\midrule
Llama-3B         & 28 & 14, 20, 24 & 14, 22, 26 & 20, 6, 22  & 16, 20, 14 & 20, 22, 12 & 14, 20, 12 & 14, 12, 6  \\
Mistral-7B       & 32 & 18, 16, 12 & 18, 30, 28 & 18, 30, 28 & 10, 8, 6   & 16, 12, 14 & 18, 28, 30 & 18, 10, 26 \\
Gemma-9B         & 42 & 28, 22, 26 & 26, 24, 22 & 24, 22, 28 & 22, 26, 24 & 38, 22, 36 & 24, 22, 26 & 28, 26, 24 \\
Mistral-24B      & 40 & 24, 26, 28 & 36, 18, 14 & 20, 24, 34 & 24, 20, 6  & 18, 6, 8   & 26, 24, 14 & 26, 14, 12 \\
Gemma-27B        & 46 & 42, 24, 30 & 22, 14, 32 & 22, 30, 18 & 24, 20, 34 & 38, 14, 36 & 24, 22, 30 & 16, 8, 32  \\
Qwen-32B         & 64 & 48, 54, 12 & 51, 54, 57 & 54, 48, 45 & 51, 48, 15 & 42, 36, 57 & 51, 54, 45 & 48, 57, 54 \\
Llama-70B        & 80 & 75, 45, 12 & 63, 39, 75 & 42, 36, 33 & 75, 12, 72 & 42, 57, 33 & 42, 75, 39 & 15, 39, 69 \\
DeepSeek-67B     & 95 & 29, 89, 92 & 44, 86, 77 & 44, 92, 35 & 14, 38, 47 & 89, 92, 14 & 44, 41, 89 & 89, 74, 47 \\
\bottomrule
\end{tabular}}

\end{table}

\FloatBarrier
\section{Control Results}
\label{app:controlresults}

\paragraph{Does the correspondence matter?}
The English rows are shuffled before fitting, so each claim is paired with the wrong English
counterpart. Same distributions, same geometry, same decomposition: only the item-level
pairing is destroyed. If this works as well, the cross-lingual alignment account is empty.

\begin{table}[ht]
\centering
\caption{Scrambled-correspondence control at depth 1, where the control mirrors the fitting
procedure exactly. Beats is the share of the three scrambled permutations that the fitted map
outperforms; $z$ is computed over those three.}
\label{tab:scrambled}
\small
\begin{tabular}{l cccc}
\toprule
Model & fitted & scrambled pairing & beats & $z$ \\
\midrule
Llama-3B         & 0.821 & 0.583 & 100\% & $+1.83$ \\
Mistral-7B       & 0.779 & 0.520 & 100\% & $+8.13$ \\
Gemma-9B         & 0.857 & 0.730 & 100\% & $+2.05$ \\
Mistral-24B      & 0.877 & 0.640 & 100\% & $+3.21$ \\
Gemma-27B        & 0.839 & 0.723 & 100\% & $+1.87$ \\
Qwen-32B         & 0.886 & 0.644 & 100\% & $+3.88$ \\
Llama-70B        & 0.879 & 0.523 & 100\% & $+3.81$ \\
DeepSeek-67B     & 0.835 & 0.719 & 100\% & $+2.31$ \\
\bottomrule
\end{tabular}
\end{table}

Every model clears it, with $z$ between $+1.83$ and $+8.13$. An earlier method, a single
language-direction steering vector, failed this same control on four of six models and was
abandoned.

\paragraph{Is it the direction, or merely a perturbation?}
Every arm is reduced to a single held-out AUROC, averaged over all eight models and all seven
non-English languages, and compared with the one number the method reaches: $0.853$, against
a do-nothing baseline of $0.820$. An arm succeeds when it lands below that. The last three
arms act on English and are compared instead with the English baseline of $0.870$, since
English is what they are scored on.

\begin{table}[ht]
\centering
\caption{Every arm, averaged over 8 models and seven languages, held out.}
\label{tab:arms}
\footnotesize
\resizebox{\textwidth}{!}{%
\begin{tabular}{l >{\raggedright\arraybackslash}p{0.40\textwidth} ccc}
\toprule
Arm & What it tests & AUROC & vs method & success \\
\midrule
baseline & nothing is applied & 0.820 & $-0.033$ & yes \\
shift only & whether moving the cloud is enough, without turning it & 0.834 & $-0.019$ & yes \\
rotation only & whether turning it is enough, without moving it & 0.844 & $-0.008$ & yes \\
unconstrained linear map & whether the orthogonality constraint earns its place & 0.786 & $-0.067$ & yes \\
scrambled pairing & whether the fact-to-fact correspondence matters, or only the two clouds & 0.621 & $-0.232$ & yes \\
random rotation & whether any rotation would do as well as the fitted one & 0.658 & $-0.194$ & yes \\
random shift & whether it is the size of the displacement or its direction & 0.819 & $-0.034$ & yes \\
random shift + rotation & whether generic perturbation of the residual stream explains the gain & 0.653 & $-0.199$ & yes \\
right map, wrong layer & whether the layer choice matters or any layer would serve & 0.793 & $-0.059$ & yes \\
inverse map, on English & whether the map is directional: English should degrade & 0.842 & $-0.028$ & yes \\
English to itself & the null: a fact paired with itself must change nothing & 0.870 & $0.000$ & yes \\
English, rows unpaired & scrambling inside a single language, with no translation involved & 0.660 & $-0.210$ & yes \\
\bottomrule
\end{tabular}}

\end{table}
\FloatBarrier

\end{document}